\documentclass[a4paper,UKenglish]{lipics}

\newcommand{\hide}[1]{}    
\DeclareSymbolFont{extraup}{U}{zavm}{m}{n}
\DeclareMathSymbol{\vardiamond}{\mathalpha}{extraup}{87}
\usepackage{changepage}
\usepackage{rotate}   
\usepackage{setspace}  
\usepackage[utf8]{inputenc} 
\usepackage{listings}
\usepackage{textcomp}  
\usepackage{multicol}
\usepackage{enumitem} 
\newcommand{\staritem}{\global\asterisktrue\item}
\newcommand{\perhapsstar}{
    \ifasterisk$\star$\global\asteriskfalse\fi
}
\newif\ifasterisk

\usepackage{bussproofs}
\usepackage{color} 
\usepackage{times} 
\usepackage{tikz}    
\usepackage{tikz-qtree}
\usepackage[framemethod=TikZ]{mdframed}
\usepackage{marvosym} 
\usepackage{listings}
\usepackage{caption}   
\usepackage{multirow} 
\usepackage{wrapfig}  
\usepackage[sans]{dsfont}
\usepackage{verbatim}
\usepackage{amssymb}   
\usepackage{amsmath} 
\hide{
\newtheorem{corollary}{Corollary}
\newtheorem{example}{Example} 
\newtheorem{proposition}{Proposition}  
\newtheorem{theorem}{Theorem} 
\newtheorem{lemma}{Lemma}  
\newtheorem{definition}{Definition}  
}
\hide{ 
\usepackage{hieroglf}
\newcommand{\hide}[1]{}    
\usepackage{rotate}    
\usepackage{graphicx} 
\usepackage{tabularx}
\usepackage[utf8]{inputenc} 
\usepackage{listings}  
\usepackage{amsthm} 
\usepackage{amsmath}
\usepackage{textcomp}   
\usepackage{multicol}
\usepackage{enumitem}
\usepackage{bussproofs}
\usepackage{color} 
\usepackage{times}  
\usepackage{helvet}  
\usepackage{courier} 
\usepackage{tikz}    
\usetikzlibrary{lindenmayersystems}
\usepackage[framemethod=TikZ]{mdframed}
\usepackage{marvosym} 
\usepackage{listings}
\usepackage{caption}   
\usepackage{multirow} 
\usepackage{wrapfig}   
\usepackage[sans]{dsfont}
\usepackage{verbatim}
\usepackage[lofdepth,lotdepth]{subfig}

\newtheorem{proposition}{Proposition}  
\newtheorem{theorem}{Theorem}       
}    
\newtheorem{proposition}{Proposition}  
\newcommand{\leftmost}{\textsf{lefm}} 
\newcommand{\Imu}{\ensuremath{\mathcal{I}_{\mu}}} 
 
\newcommand{\update}{\ensuremath{\zeta}} 
\newcommand{\properO}{\ensuremath{\textsf{properO}_{\mathfrak{D}}}}
\newcommand{\epgl}{\textsc{EPGL}}
\newcommand{\fixpoint}{\textsf{fix}}
\newcommand{\argument}{\textsf{arg}}

\newcommand{\obullet}{\ensuremath{o^{\bullet}}} 
\newcommand{\ostar}{\ensuremath{o^{\star}}} 
\newcommand{\proper}{\ensuremath{\textsf{proper}_{\mathcal{I}_{\mu}}}}
\newcommand{\completely}{\textsf{completely}}
\newcommand{\something}{\textsf{something}}
\newcommand{\pnode}{\textsf{Pnode}}
\newcommand{\Subformula}{\textsf{Subformula}} 
\newcommand{\Subobject}{\textsf{Subobject}}

\newcommand{\Subvariable}{\textsf{Subvariable}}
\newcommand{\Subterm}{\textsf{Subterm}}
\newcommand{\Outer}{\textsf{Outer}}
\newcommand{\TScompatible}{\textsf{TScompatible}}

\newcommand{\term}{\ensuremath{\mathbf{t}}} 

\newcommand{\Plato}{\textsf{Plato}}  
\newcommand{\shelf}{\textsf{shelf}}

\newcommand{\orC}{\textsf{or}}
\newcommand{\andC}{\textsf{and}}

\newcommand{\notC}{\textsf{not}}
\newcommand{\ryuta}[1]{{\color{red}{[#1]}}}
 
\newcommand{\tbullet}{\ensuremath{t^{\bullet}}}
\newcommand{\tstar}{\ensuremath{t^{\star}}} 
\newcommand{\tstarLength}[1]{\ensuremath{t^{\star{\small\fbox{#1}}} }}
\newcommand{\ostarLength}[1]{\ensuremath{o^{\star{\small\fbox{#1}}} }}
\newcommand{\Bound}{\textsf{Bound}}
\newcommand{\Free}{\textsf{Free}}

\newcommand{\compatible}{\textsf{Tcompatible}}
\newcommand{\AQcompatible}{\textsf{AQcompatible}}
\newcommand{\Ccompatible}{\textsf{Ccompatible}}
\newcommand{\match}{\ensuremath{\textsf{match}}}
\newcommand{\al}{\ensuremath{\alpha}} 
\newcommand{\alstar}{\ensuremath{\alpha^{\star}}} 
\newcommand{\albullet}{\ensuremath{\alpha^{\bullet}}}

\newcommand{\isPart}{\textsf{isPartOf}}
\newcommand{\kick}{\textsf{kick}}
\newcommand{\door}{\textsf{door}}
\newcommand{\John}{\textsf{John}}
\newcommand{\down}{\textsf{down}}
\newcommand{\girl}{\textsf{girl}} 
\newcommand{\eyes}{\textsf{eyes}} 
\newcommand{\brown}{\textsf{brown}}

\newcommand{\normal}{\textsf{normal}}
\newcommand{\farm}{\textsf{farm}} 
\newcommand{\new}{\textsf{new}}

\newcommand{\rabbit}{\textsf{rabbit}} 
\newcommand{\elephant}{\textsf{elephant}}
\newcommand{\Socrates}{\textsf{Socrates}}
\newcommand{\illReal}{\textsf{illReal}}

\newcommand{\like}{\textsf{like}} 
\newcommand{\order}{\textsf{allowOrderThrough}}
\newcommand{\employ}{\textsf{employ}}
\newcommand{\use}{\textsf{use}}
\newcommand{\school}{\textsf{school}} 
\newcommand{\handsome}{\textsf{handsome}}  
\newcommand{\student}{\textsf{learner}}
\newcommand{\pcinstructor}{\textsf{pcinstructor}} 
\newcommand{\teach}{\textsf{teach}} 
\newcommand{\female}{\textsf{female}}
\newcommand{\Mac}{\textsf{Mac}}
\newcommand{\buy}{\textsf{buy}}
\newcommand{\has}{\textsf{has}}
\newcommand{\father}{\textsf{father}}
\newcommand{\daughter}{\textsf{daughter}}
\newcommand{\speak}{\textsf{speak}}
\newcommand{\lang}{\textsf{language}}

\newcommand{\walk}{\textsf{walk}}
\newcommand{\run}{\textsf{run}}
\newcommand{\teacher}{\textsf{teacher}}

\newcommand{\beat}{\textsf{beat}}
\newcommand{\donkey}{\textsf{donkey}} 
\newcommand{\own}{\textsf{own}} 
\newcommand{\farmer}{\textsf{farmer}}
\newcommand{\RNum}[1]{\uppercase\expandafter{\romannumeral 
        #1\relax}}

\newcommand{\Green}{\textsf{Green}}

\title{Predicate Gradual Logic and
    Linguistics}   
\author{Ryuta Arisaka}  
\affil{National Institute of Informatics\\
ryutaarisaka@gmail.com} 
\date{}
\begin{document} 
\maketitle 
\begin{abstract}  
   There are several major proposals
   for treating donkey anaphora  
   such as discourse representation theory 
   and the likes, or E-Type theories 
   and the likes. 
   Every one of them  
   works well for a set of specific examples 
   that they use to demonstrate validity 
   of their approaches. 
   As I show in this paper, however, they 
    are 
   not very generalisable 
   and do not account for essentially the 
   same problem that they remedy  
   when it manifests in other examples. 
    I propose another  
  logical approach. 
  I develop 
  logic that extends a recent, propositional 
  gradual logic, and show that it 
  can treat donkey anaphora generally. 
  I also identify and address 
  a problem around the modern 
  convention on existential import.  
  Furthermore, I show that Aristotle's syllogisms 
  and conversion are realisable 
  in predicate gradual logic. 
\end{abstract} 
\section{Introduction}        
\begin{enumerate}[label={(\arabic*)}] 
\item Every farmer who owns a donkey beats it.
\end{enumerate} 
The donkey sentence by Geach 
\cite{Geach62} depicts 
how Fregean predicate logic fails   
to naturally encode anaphoric expressions. 
\begin{enumerate}[label={(\protect\perhapsstar\arabic*)}] 
\setcounter{enumi}{1}
\item There is a farmer who owns a donkey. It runs. 
\item[] {\small $\exists x \exists y.
    \farmer(x) \wedge \donkey(y) 
    \wedge \own(x, y) \wedge \run(y)$}  \\
\item Every farmer owns a donkey. It runs. 
\setcounter{enumi}{2}
\item[] {\small $\forall x.(\farmer(x) 
\supset \exists y.\donkey(y) \wedge \own(x, y) \wedge 
\run(y))$}\\ 
\end{enumerate}  
    Before going further, I state my assumption that  
    $\neg$ binds the strongest, 
    $\wedge$ and $\vee$ bind 
    the second strongest, 
    $\forall$ and $\exists$ the third strongest, 
    and $\supset$ the weakest. The expressions (2) and (3) 
and the orthodox first-order encoding 
shown below them suggest that the indefinite noun {\it a donkey}, 
in the presence of the pronoun {\it it} anaphoric to it, 
should be understood to be possessing existential force. 
However, the heuristics does not 
appear to be extendible to (1), for, suppose 
it were applicable, 
we would be 
getting the following first-order 
encoding 
which is not a sentence. \\\\
$\forall x.( 
\farmer(x) \wedge (\exists y.\donkey(y) \wedge 
\own(x, y)) \supset \beat(x, y))$.    \\\\
We also cannot 
just bring the existential 
quantifier out of the inner scope. \\\\
$\forall x \exists y.(\farmer(x) \wedge 
\donkey(y) \wedge \own(x, y) \supset 
\beat(x, y))$. \\\\
In this attempted encoding, 
assume for now that each entity 
$d_1, d_2, \ldots$ 
in a given domain of discourse 
is denoted by a constant 
$a_1, a_2, \ldots$ in a one-to-one manner, 
then $\exists y$ binds 
$\farmer(a) \wedge \donkey(y) \wedge \own(a, y) 
\supset \beat(a, y)$ for each such $a$. 
Hence we only need find some individual  
denoted by some constant $a_2$ 
for every such $a$ 
to replace $y$ with in the expression 
such to make  
$\farmer(a) \wedge \donkey(a_2) \wedge \own(a, a_2) 
\supset \beat(a, a_2)$ true. 
Suppose we have a domain of discourse 
having three entities: a farmer, a donkey 
and a pebble denoted by $a_1$, by $a_2$ 
and respectively by $a_3$ such that 
the farmer owns the donkey which 
he/she does not beat.  
If $x$ is replaced by $a_2$ or $a_3$, 
the antecedent of the material implication 
is false, and the sentence is vacuously 
true,  
and if $x$ is replaced by $a_1$, then substitution of  
$a_3$ into $y$ makes it true. So 
the first-order sentence, quite out of keeping 
with the semantics of (1), is true in the given domain 
of discourse. \\
\indent These taken into account, 
formal semanticists who accept 
the view that the Geach' sentence 
is first-order formalisable 
in a standard way 
are forced 
to choose the following 
alternative for the representation 
of (1). 
\begin{enumerate}[label=(\roman*)]  
\item 
$\forall x \forall y.(\farmer(x) \wedge 
\donkey(y) \wedge \own(x, y) \supset 
\beat(x, y))$. 
\end{enumerate} 
Now, if it were the adequate encoding, 
a pair of 
an indefinite noun and a pronoun anaphoric 
to it would be sometimes understood existentially,  
sometimes understood universally,  
the quantificational force  
irregularly varying depending 
on the context in which it is found. 
That would be certainly a hindrance to 
automating formalisation of 
anaphoric expressions.  \\
\indent A quick search on the Web has 
taken the author to the Wikipedia, the Stanford
Encyclopedia of Philosophy, and 
a few forum threads on the donkey anaphora. 
Except for the Wikipedia entry where multiple 
viewpoints are blended in, all the others  
seem to assert that  
{\it it is not the case that the donkey sentence 
cannot be encoded as a first-order expression 
in a standard way; it is 
just that the encoding of the donkey sentence in 
that manner is inconsistent 
with the encoding of other anaphoric expressions}, 
which can be taken to be 
a consensus among - at the very least - the users 
of formal logic.\\
\indent However, regardless of the trend, 
I cannot help  wondering 
how (i) came to be considered 
the adequate encoding of (1), i.e. 
since when it is the case that (1) means (4). 
\begin{enumerate}[label=(\arabic*)]  
\setcounter{enumi}{3} 
   \item Every farmer who owns a 
   donkey 
   beats 
   every one of the donkeys he/she owns.
\end{enumerate} 
Suppose a domain of discourse 
in which there are two farmers. 
One of them owns a couple of donkeys which 
he/she beats equally. 
The other owns two donkeys and beats 
one of them but not the other. While  
there is hardly any doubt that (4) is false in this domain of discourse, 
the judgement that 
(1) must be false in the same domain of discourse 
sounds a little too audacious. 
When we utter (1) correct to 
its usage, we surely cannot mean  
so much more off the mark as (4) than the following: 
\begin{enumerate}[label=(\arabic*)]  
        \setcounter{enumi}{4}
    \item Every farmer who owns a donkey 
        beats a donkey he owns. 
\end{enumerate}
which, as I will argue like a linguist 
in cognitive tradition 
may do,  already 
carries imprecision, but which 
may be still acceptable 
by force of common sense prevailing over this 
particular donkey sentence. In any case, 
logic should capture general principles 
rather than specific appearance of 
them. 
\subsection{Previous Approaches}   
There are several proposals 
for formally treating the donkey anaphora such as  
discourse representation theory 
and the likes
\cite{Heim82,Kamp11,Kamp81,Groenendijk91,Muskens96}, 
or E-Type theory and the likes \cite{Wilson84,Neale90,Davies81,Evans77}. Still, it does not appear as if 
any of them had extricated  
the linguistic problems that the donkey anaphora presents. 
There is a defeating counter-example 
to all of them. 
\begin{enumerate}[label={(\protect\perhapsstar\arabic*)}] 
        \setcounter{enumi}{5}
    \item Every school which employs a 
        handsome PC instructor  
         who teaches every 
         female learner using a Mac machine 
         buys it for him. 
\end{enumerate} 
It is just another donkey sentence having 
the same linguistic characteristic 
as the Geach' sentence. 
If it were the case 
that they had indeed solved the donkey 
anaphora, then the same formal techniques 
could adequately encode (6) just as 
they could (1). 
But none of them can formalise (6). 
\subsubsection{Dynamic Theories}   
Appeared earliest - or at least published earliest - in this movement 
is discourse representation theory (DRT). 
Unlike  more traditional formal linguistics 
formed around the Russelian interpretation,  
DRT treats quantifying 
force of a noun phrase context-dependently. 
In classic DRT, it is possible to derive (i) 
from (1) in a manner consistent 
to `more normal' anaphora. 
In the more recent DRT, 
it is possible 
to interpret (1) as (ii).   
\begin{enumerate}[label=(\roman*)] 
        \setcounter{enumi}{1} 
    \item      For 
        every farmer $X$ who owns a donkey $Y$,  
        it holds true that $X$ beats $Y$.  
\end{enumerate}
If one accepts the commonly approved view  
\cite{Kanazawa94} 
that (1) has both the so-called weak reading: 
(5), and the so-called strong reading: (4), then 
the modern DRT somewhat 
offers an interim solution in 
 that its user, according to his/her 
 preference, could choose 
either (4), or (ii) -  which 
is not quite (5) but still closer 
to it than (4) is. Even then, 
the revamped DRT does not voluntarily 
facilitate two different 
formal semantics corresponding to them, 
one for each: suppose 
that somebody has encoded (4) using the modern DRT; 
suppose that we pass the formal expression 
to Mr X, but not the fact that 
it is meant to be formalisation of (4); 
suppose that Mr X wants to find 
a natural expression corresponding to 
the formal expression; then 
he may be inclined to say that (4) corresponds to it, 
that (ii) corresponds to it, or,
if he is cautious, that it could be 
(4) or (ii); but Mr X cannot obtain that it must be 
(4). That is, DRT formalisation 
of a donkey sentence is not lossless,
multiple donkey sentences in natural language mapping 
to a DRT expression in general. 
Besides, neither the classic nor the modern DRT 
can encode (6).  If we apply the more recent DRT to (6), 
what we eventually get is the following 
expression:\footnote{
    Readers 
    are referred to the cited references 
    for the syntax.}\\\\ 
{\small $[_0[_1x,y:\school(x),
    \handsome(y),\pcinstructor(y), \employ(x,y), \\$}
{\small $
    [_2 z, w:\female(z), \student(z), 
    \Mac(w), \use(z,w)]\forall z
    [_3 \teach(x,z)]]\forall x
    [_4 u: \buy(u, x)]]$}. \\\\
\hide{
    $[_0[_1x,y:\farmer(x),\daughter(y),[_2 z: \lang(z)]\forall z[_3\speak(y,z)]]\forall x[_4 u: 
    \speak(x,u)]]$}
In this DRT representation $u$ cannot be resolved
to be $w$ as DRS 2 that contains the variable $z$ 
is not accessible from 
DRS 4 that needs it. \\
\indent Neither can FCS which is another 
quantifier-free theory which puts 
an extra focus on discourse referents and 
the relations to hold between them,
nor DPL, a formal variant of DRT and FCS, 
encode (6). First of all, the principle 
of FCS, 
simply for the fact that it encodes (1) into (i), 
is controversial. 
DPL can translate (1) into (ii), on the other 
hand. However, DPL, and 
also FCS, 
suffer from the same scoping issue as DRT 
does. 
\subsubsection{Substitution theories}
There are other approaches based on 
E-Type Theory by Evans \cite{Evans77}. \begin{enumerate}[label={(\protect\perhapsstar\arabic*)}] 
        \setcounter{enumi}{6} 
    \item If Thomas$_1$ owns a donkey$_2$, 
        he$_1$ beats it$_2$.   
    \item If Thomas owns a donkey, 
        {\it (the) Thomas} beats 
    {\it the donkey}. 
\end{enumerate}
Evans  
argues that, if for instance we 
have (7),  
it should be paraphrasable 
to (8) where 
the {\it he} anaphoric to   
{\it Thomas} 
is replaced 
by that particular {\it Thomas} who owns 
the donkey, and similarly for the {\it it}.   
By applying this principle 
to (1), we obtain (9). 
\begin{enumerate}[label=(\arabic*)]
        \setcounter{enumi}{8} 
    \staritem Every farmer 
    who owns a donkey beats {\it 
        the donkey.} 
\end{enumerate} 
The uniqueness presupposition  
justifies the use of 
the definite noun phrase {\it the donkey} 
in (9). Despite the E-Type theory having 
appeal to intuition, this approach 
has not found many proponents. There is 
firstly an obvious doubt over the uniqueness 
condition. There is, however, more 
serious a problem: it requires 
the substituting entity to carry 
the context in which {\it a donkey} was used. 
This cannot be the case in general, as I will 
discuss shortly. \\
\indent The so-called numberless theories  
\cite{Davies81,Neale90} 
derive from the E-Type Theory. For 
the perceivable inconvenience in 
the uniqueness presupposition, those 
theories interpret (1) as follows. 
\begin{enumerate}[label=(\arabic*)]
        \setcounter{enumi}{9} 
    \staritem Every farmer who owns a donkey 
        beats the donkey or the donkeys 
        he/she owns. 
\end{enumerate} 
These theories, as derivatives 
of the E-Type theory, 
inherit the serious problem 
from it. 
Kanazawa argues against the semantic 
plurality these theories  assume  \cite{Kanazawa01}. 
\subsubsection{Anaphora is not bound by the context 
    that binds its antecedent}    
There is a problem common to all that have been 
mentioned above. Consider (6), the Mac sentence, again.  
In the sentence, while the pronoun {\it it}  
is anaphoric to the {\it Mac machine}, 
it is by no means the case that 
the Mac sentence intends to mean that 
School X which satisfies the stated condition 
buys the Mac machine(s) owned by the female 
learner(s) for the handsome PC instructor.  
Most naturally, it is some brand new Mac machine 
the school buys for him. In another word, 
while the {\it it} is indeed some {\it Mac machine}, 
the {\it Mac machine} cannot be further made 
specific by the specific context in which 
it appears in the Mac sentence. \\
\indent Hence the Mac sentence is a testimony 
also to the fact 
that 
neither the strong reading of (1), which is 
(4), nor the weak reading of it, which is (5), 
as categorized in \cite{Kanazawa94} is  
general enough for a general donkey sentence. 
See \cite{Fauconnier85} for effectively 
the same outlook as the author's.
\subsection{A new approach with Gradual Logic}      
I have shown that, if the effect of 
the donkey anaphora has been 
somehow addressed for specific donkey sentences  
by the above-mentioned formal approaches, 
the cause of it has still not been.  
In this work I will identify a fundamental 
cause of the gap 
between natural and formal languages, 
and, by resolving it, solve 
the donkey anaphora within formal logic 
as a consequence. \\
\indent To address 
its cause, it is necessary that we contemplate over 
the following point. \begin{itemize} 
\item There are no smallest concepts 
    referring to objects 
in natural languages, {\it so long just as 
we acknowledge infinitely extending notions 
such as space, time, size, colour, etc, 
in short any reasonable quality into which 
an object is subjected} \cite{Arisaka15Gradual}.  
But formal/symbolic logic has always assumed 
the entities 
that cannot be 
further divisible. 
\end{itemize}       
Hence inspecting the logical atomicism and 
developing 
an alternative foundation for formal/symbolic 
logic should be the key to closing 
the gap.   I presented a class of propositional 
logics that do not assume smallest entities 
\cite{Arisaka14tech1,Arisaka15Gradual}. 
I will 
derive a solution to the donkey anaphora  
by extending, with quantifiers, one formulation of propositional 
gradual logic.  \\
\indent At least something should be said 
of gradual logic, however. Propositional gradual logic 
bears a deceptively simple appearance. Apart from 
the usual logical connectives available 
to propositional logic, there is only 
one additional logical connective $\gtrdot$. 
From specific examples used in \cite{Arisaka14tech1,Arisaka15Gradual}, 
I recall 
$\textsf{Hat} \gtrdot \textsf{Green}$ 
for instance. Given, as so supposed 
in \cite{Arisaka14tech1,Arisaka15Gradual}, 
that  
$\textsf{Hat}$ and $\textsf{Green}$  
both convey existence of some entities, it means 
that Hat is, and Green as an attribute of Hat 
is. It moreover means that 
$\textsf{Hat} \gtrdot \textsf{Green}$ 
itself indicates existence of a concept: Green Hat, if 
it is indeed the case that the concept 
Hat has the attribute of being Green.  
The usual nullary 
logical connectives of propositional logic $\top$ 
and $\bot$ signify 
the presence and the absence respectively. 
A postulate in \cite{Arisaka15Gradual} 
dictates that 
for any concept $\textsf{X}$ it always holds 
that $\textsf{X} = 
(\textsf{X} \gtrdot \top)$, i.e. a concept 
is if and only if it is such that 
it has some attribute, which implies 
that there cannot be an atomic entity 
in Gradual Logic.

To speak 
more of specific details 
about what the attribution relation means 
here semantically, we need 
to keep in mind that 
Gradual Logic assumes 
multiple domains of discourse. 
Let us say that any concept that appears 
at the leftmost position in 
some $\gtrdot$ chain; e.g. 
$\textsf{Hat}$ in $\textsf{Hat} 
\gtrdot (\textsf{Green} \gtrdot \textsf{Lamination})$, 
a 0-th degree concept. Any 0-th degree concept 
is a constituent of  some domain of discourse. 
The $\textsf{Green}$ in 
$\textsf{Hat} \gtrdot (\textsf{Green} 
\gtrdot \textsf{Lamination})$ is 
in another domain of discourse, whose 
existence is dependent 
on the existence of $\textsf{Hat}$ 
in the 0-th degree domain of discourse. 
In the like manner,  
$\textsf{Lamination}$ in 
$\textsf{Hat} \gtrdot (\textsf{Green} 
\gtrdot \textsf{Lamination})$ 
is in another domain of discourse 
which depends on 
the existence of $\textsf{Green}$ 
in the domain of discourse 
which in turn depends on the existence 
of $\textsf{Hat}$ in the 0-th degree domain 
of discourse. \\
\indent Although propositional gradual logic 
does not handle 
quantified expressions generally, 
there is a sketch towards Gradual Predicate Logic 
in the last section 
of \cite{Arisaka15Gradual}. The key idea 
mentioned is to 
consider a propositional gradual logic 
expression as a term; e.g. 
instead of using a variable $x$ 
to say there is a green hat: 
$\exists x.\Green(x) \wedge \textsf{Hat}(x)$, 
we let a predicate range 
over a propositional gradual logic expression: 
$\textsf{is}(\Green \gtrdot \textsf{Hat})$. 
However, given that 
there is virtually no predicate-term 
distinction in propositional gradual logic, 
it should be also possible to 
use a gradual logic expression as 
a predicate. I will show that such idea can be 
indeed formalised, which I may 
call Predicate Gradual Logic for
some distinction from the idea of gradual predicate logic 
in the mentioned reference. 
There is obviously no difficulty involved 
in augmenting propositional gradual logic with quantifiers for 
the concepts in the 0-th degree 
domain of discourse.  They 
appear basically the same as the standard 
first-order expressions. 
\begin{enumerate}[label=(\roman*)] 
    \item[] There is some 
        hat. 
        \setcounter{enumi}{2}
    \item $\exists x.\textsf{Hat}(x)$.\\
    \item[] Every man walks. 
        \setcounter{enumi}{3}
    \item $\forall x.\textsf{Man}(x) \supset 
        \walk(x)$. 
\end{enumerate}
Unrestricted predication 
over a general gradual logic expression appears, 
by contrast, 
a lot more involved. 
Nonetheless, recall 
that $\textsf{Hat} = ({\textsf{Hat} \gtrdot \top})$, 
which means that (iii) is the same as 
$\exists x.({\textsf{Hat} \gtrdot \top})(x)$, 
while (iv) is the same as
$\forall x.({\textsf{Man} \gtrdot \top})(x) 
\supset ({\walk \gtrdot \top})(x)$. Hence, 
at least in a certain way, (iii) and (iv) already 
illuminate a solution to the difficult cases. 
In case we have an expression 
{\it There is a hat which is green.}, 
the corresponding predicate gradual logic encoding 
should be something to do with 
replacing $\top$ in (iii) with 
an expression that involves $\textsf{Green}$. 
I develop formal semantics of 
Predicate Gradual Logic based on this insight.
\subsection{Predicate Gradual Logic encoding 
    of donkey sentences} 
I believe, however, that it is probably 
more helpful if I begin by presenting 
examples of predicate gradual logic expressions 
than by spelling out rigorous formal definitions 
outright. Semantics will not be developed 
until Section 3, and the formal expressions 
themselves may not be fully comprehended 
except by some. But it is more important to 
understand explanations of features 
of logic given after the examples at this point 
than the examples themselves. 
Now, although I am presenting materials 
in this order 
for the best, any reader who wishes 
to see examples 
only after the semantics of logic has been given 
should immediately go to 
Section 3, and should come back to this 
part and Section 2 later. \\\\
 \textbf{Minimal reading aid}: 
    $x$, $y \lessdot x$, and $u \lessdot x$ and 
    so on are variables. $\lessdot$ 
    is left-associative. $\gtrdot$ binds 
    stronger than $\forall$ and $\exists$ but weaker 
    than $\wedge$ and $\vee$. $\simeq$ 
    is a binary operator. 
\begin{enumerate}[label={(\protect\perhapsstar\arabic*)}]  
        \setcounter{enumi}{4} 
    \item[(5)] Every farmer who owns a donkey beats 
        a donkey he/she owns. 
        \setcounter{enumi}{4} 
        \staritem $\forall x \exists y \lessdot x 
        \exists u \lessdot x.
    ([\farmer \gtrdot \donkey(y) \wedge 
    \own(^\circ x,y)](x) \supset  {y \lessdot x}
    \simeq {u \lessdot x} 
    \wedge 
    {\beat(x, u \lessdot x)})$.  
    \\
    \setcounter{enumi}{3} 
    \item Every farmer who owns a donkey 
        beats every one of the donkeys 
        he/she owns. 
        \setcounter{enumi}{3}
        \staritem $\forall x \forall y \lessdot x.
        ([\farmer \gtrdot (\donkey(y) \supset 
    \own(^\circ x,y))](x) \supset 
    \beat(x,y \lessdot x))$. \\ 
    \setcounter{enumi}{0}
\item Every farmer who owns a donkey 
    beats it.  
    \setcounter{enumi}{0} 
\staritem $\forall x \exists y \lessdot x\ \exists u.
    ([\farmer \gtrdot \donkey(y) \wedge 
    \own(^\circ x, y)](x) \supset  
    y \simeq u \wedge 
    \beat(x, u))$. \\
    \setcounter{enumi}{5} 
    \item Every school who employs a handsome 
        PC instructor who teaches every female 
        student using a Mac machine buys 
        it for him. 
        \setcounter{enumi}{5}  
        \staritem 
        {\small $\forall x\ \exists y \lessdot x \ 
            \forall z \lessdot y \lessdot x\ \exists w \lessdot z \lessdot y \lessdot x\ \exists u.([\school \gtrdot 
            ([\pcinstructor \gtrdot \handsome(^\circ y) \wedge ([\student 
            \gtrdot \female(^\circ z) \wedge 
            (\Mac(w) \wedge  
            \use(^\circ z, w))
        ](z) 
            \supset \teach(^\circ y,z))
            ] (y) \wedge \employ(^\circ x,y) )]
            (x) \supset w \simeq u \wedge  
            \buy(x, u))$}. \\
        \setcounter{enumi}{10} 
    \item  Every school who employs 
    a handsome PC instructor who teaches 
    every female student who 
    likes a Mac allows her to 
    order it through him. 
   \setcounter{enumi}{10}
   \staritem  {\small $\forall x\ \exists y \lessdot x \ 
            \forall z \lessdot y \lessdot x \ \exists w \lessdot z \lessdot y \lessdot x\ \exists u.([\school \gtrdot 
            ([\pcinstructor \gtrdot \handsome(^\circ y) \wedge ([\student 
            \gtrdot \female(^\circ z) \wedge 
            (\Mac(w) \wedge  
            \like(^\circ z, w))
        ](z) 
            \supset \teach(^\circ y,z))
            ] (y) \wedge \employ(^\circ x,y) )]
            (x) \supset  
            w \simeq u \wedge 
            \order(x, z, u, y))$}.
\end{enumerate}  
\subsubsection{Scopes of quantifiers}  
In first-order logic, supposing 
that $x$ is some variable, we have $\forall x$ for 
universal quantification, and $\exists x$ 
for existential quantification. 
In either of the cases, 
$x$ does not occur free in the expression 
in the scope of it. 
That is 
no different in predicate gradual logic.  
\subsubsection{Attributed objects} 
But a variable in predicate gradual logic 
is not assumed to be indivisible. In 
all the five formal expressions above there is  
some variable of the sort of $y \lessdot x$.  
It means $y$ as an attribute of $x$.  \\
\indent Significance of $\simeq$ is best explained 
with the recognition cut-off principle 
\cite{Arisaka15Gradual,Kant08}: even if two things are 
distinguishable, they are not so distinguished in 
a domain in which the distinction cannot be perceived. 
$v_1 \simeq v_2$ says that $v_1$ may not 
be the same as $v_2$ but that 
they are indistinguishable as far as 
the most {\it salient parts} of $v_1$ and $v_2$, 
like foreground is to backgrounds in Langacker's 
cognitive grammar \cite{Langacker87,Langacker91}, 
are concerned. Such surficial indistinguishability 
is amply expected from the fact 
that there are many 
domains of discourse in gradual logic. 
For instance when we have 
$y \lessdot x$ and $u \lessdot x$,
the entity referred to by $y$ and 
the entity referred to by $u$ are being `looked in' by $\lessdot$ as 
an attribute of the entity referred to by $x$, which are then 
the most salient parts of $y \lessdot x$ and of 
$u \lessdot x$ respectively.  
The expression $y \lessdot x \simeq u \lessdot x$ 
signifies that 
the entity for $y$ 
is indistinguishable from the entity for $u$ within 
the particular domain of discourse in which 
the two are discussed. See ($\star$5), 
and take a contrast with ($\star$4). \\
\indent In this work, I consider 
{\it attribute normal interpretation} 
for attribution. 
Under the interpretation, 
that the entity in a domain 
of discourse $\mathbf{Y}$ is an attribute 
of the entity 
in a domain of discourse $\mathbf{X}$ 
means that $\mathbf{X}$ contains the latter  
as well as the former entities. 
Hence $\exists x \lessdot a_1.\something(x \lessdot a_1)$ for instance 
requires that some entity referred to by 
$x$
be found in at least two domains of discourse:  
in the domain of discourse 
as an attribute domain of the entity referred to by 
$a_1$; 
and in the domain of discourse 
in which the entity is found. See $y \simeq u$ in ($\star$1), 
to the right of $\supset$ which is 
bound indirectly by $\exists y \lessdot x$ 
in this manner. The formal expression 
does not limit the donkey to be beaten 
to be some donkey owned by a farmer, although 
that too is a possibility. The portion: 
$y \simeq u \wedge \beat(x, u)$ merely 
ensures that the donkey be a donkey. \\
\subsubsection{Existential import}  
Under Fregean interpretation, a universal 
expression does not assume existential import, 
while an existential expression does. 
An expression of the sort of (1), under 
the interpretation, does not presuppose 
existence of a farmer who owns a donkey. 
That is, it can be a phantom term. 
I will reason in Section 2 that this convention 
does not assign a truth value clearly 
to natural expressions, e.g. the Mac sentence. 
For now I note that 
there should not be any such phantom terms
if the truth/falsehood of the sentence that 
contains them is judgeable at all. 
Consequently, for any predicate available 
to any domain of discourse, there must be 
at least one entity in the domain of discourse 
to be predicated by it.\\
\subsubsection{Attributed predicates} 
Not only objects, but also predicates 
can have attributes in predicate gradual logic.  
For ($\star$4), 
{\small $\farmer \gtrdot \textsf{X}$}
is a predicate: is-farmer-with-the-attribute-$\textsf{X}$.
If $\textsf{X}$ is $\top$, {\small $(\farmer \gtrdot \textsf{X})
    = (\farmer \gtrdot \top) = \farmer$} according 
to the principle of Gradual Logic, 
and the predicate is is-farmer-with-some-attribute, 
or is more simply is-farmer. This 
means that 
$\farmer \gtrdot \exists y.\donkey(y)$ is a predicate 
is-farmer which takes $\exists y.\donkey(y)$ 
as its attribute. Given, however, that there should be many specific attribution relations, 
we may explicitly state the own-owned relation 
in ($\star$4) 
on $\gtrdot$: 
$\farmer \gtrdot_{\own} \exists y.\donkey(y)$. 
Nonetheless, instead of diversifying 
$\gtrdot$, deriving a multi-$\gtrdot$ gradual logic that 
way, we may agree on a convention 
that, by  
$[\farmer \gtrdot \exists y.\own(^\circ a, y) \wedge \donkey(y)](a)$, 
the $^\circ a$ in $\own(^\circ a, y)$ 
refers to the entity to which $\own(^\circ a, y)$ 
is an attribute, 
$\farmer$ in this case. Then the expression 
is 
a paraphrase of 
$[\farmer \gtrdot_{\own} \exists y.\donkey(y)](a)$ 
but  with 
actually a greater mathematical and philosophical flexibilities. 
Now, although is-farmer-who-owns-a-donkey 
is more descriptive 
than is-farmer, 
it is no less a predicate 
than is-farmer is, 
the former (should be) requiring the same number of arguments as the latter. 
Provided that the entity denoted by $a$ is in the given domain of discourse, 
$[\farmer \gtrdot \exists y.\donkey(y) \wedge 
\own(^\circ a, y)](a)$ denotes 
the assertion that what 
$a$ denotes is a farmer who owns a donkey. 
By  
$\forall x.([\farmer \gtrdot \exists y.\donkey(y) 
\wedge \own(^\circ x, y)](x) \supset \mathsf{X})$, 
we express 
the following assertion: 
for every 
farmer who owns a donkey, it is the case that $\mathsf{X}$, or, if the scope of $y$ must extend to 
$\mathsf{X}$ as well, 
then $\forall x \exists y \lessdot x. 
([\farmer \gtrdot \donkey(y) 
\wedge \own(^\circ x, y)](x) \supset \mathsf{X})$, 
which provides a 
generic solution to (1), (4) and (5), 
as ($\star$1), ($\star$4) and ($\star$5).  \\
\indent Now, at this point careful readers may fret that, 
if we adopt this kind of predication, 
an entity to take place of $y$ as 
an attribute to an entity to take  place 
of $x$ 
must be discussed 
before the latter can be ever determined. 
To assure readers, 
I mention that that is just an ostensible paradox that  
does not need to be worried about. 
I will provide details in Section 2. 
\subsubsection{Let us check the given examples}  
That (1) is different from 
(4) is formally expressed. 
Further, it does not matter how deep 
in a noun phrase a noun phrase to be referred 
back by a pronoun occurs for 
the predicate gradual logic encoding to work.\footnote{ 
    This is not to say that  
    this first-order predicate gradual logic 
    could encode all the donkey sentences.}  
Take (1), it says that, for any 
farmer who owns a donkey, 
it is the case that he/she beats a donkey; Cf. 
the last section of 1.1.
But this is precisely 
what ($\star$1) is expressing. 
It almost makes no difference if 
we instead like to express (4), to which 
predicate gradual logic provides ($\star$4).  
In a like manner, yet another distinct 
expression is obtained for 
the natural expression below that 
perhaps was the true intention of Evans' 
in tackling the Geach' sentence. 
\begin{enumerate}[label={(\protect\perhapsstar\arabic*)}]   
        \setcounter{enumi}{11}
    \item Every farmer who owns a donkey 
        beats the donkey he/she owns.   
        \setcounter{enumi}{11}
    \staritem $\forall x \exists y \lessdot x.
    ([\farmer \gtrdot \donkey(y) \wedge 
    \own(^\circ x,y)](x) \supset {\beat(x, y \lessdot 
        x)})$. 
\end{enumerate}  
The innovation of predicate gradual logic 
in view of natural expression encoding 
is that 
it can express those 
different expressions distinctly, allowing 
loss-lesser translation of anaphora than 
is possible in currently predominant 
formal approaches. \\
\indent The predicate gradual logic 
encoding of (6) appears a little daunting. 
However, at every attributive domain of discourse we recognise 
self-similarity. Technically, therefore, 
($\star$6) is only a longer version 
of ($\star$1) and ($\star$4) - except that  
we cannot neglect the matter 
of existential import of a universally 
quantified expression. This, as I mentioned, is to be discussed 
in Section 2. \\ 
\indent The predicate gradual logic 
encoding of (11) appears even more frightening. 
Not only existentially quantified 
entities but also a universally 
quantified entity are referred to  
in {\it allows her to order it through him}. 
Actually handling this kind of sentence is 
beyond the scope of this work. Just like  
Evans' counter-example to Geach thesis on anaphora: 
\begin{enumerate}[label=(\arabic*)]  
        \setcounter{enumi}{12}
    \item Harry brought some sheep and 
        John vaccinated them. 
\end{enumerate}
which says (most likely) that John vaccinates 
all the sheep brought by Harry, 
there are donkey sentences of this kind 
in which all that satisfy a condition (e.g. 
all the female students who like a Mac machine 
in (11)) need to be treated as though 
they are a single entity. 
So I am only listing it to demonstrate 
that first-order 
predicate gradual logic can deal with 
this kind of an involved natural expression 
in some cases. In 
 ($\star$11) it is the case that 
possibly many entities 
that match the 
formal description   
$[\student 
\gtrdot \female(^\circ z) \wedge 
            \exists w.(\Mac(w) \wedge  
            \like(^\circ z, w))
        ](z)$ should be most naturally treated 
        if they were talked about collectively. 
At the same time, however, it is 
certainly possible to think of the elements 
in the collection one by one, and the 
method works for (11). 
\section{On certain philosophical matters}   
Predicate gradual logic does not share 
the same philosophical foundation with Fregean 
logic. I will discuss on two points. 
\subsection{Existential import}   
\subsubsection{Being}  
\begin{enumerate}[label=\arabic*)]  
\setcounter{enumi}{0}
\item Man walks. 
\item Tom walks. 
\end{enumerate}  
Many perhaps agree that {\it Man} in 1) 
is a man in general, or idealisation
of what a man is, talked as though a single unique entity, 
if, as Aristotle considers, an expression that is 
of universal character but that is not expressed 
universally should be distinguished 
from a universal expression of universal 
character. I reckon, however, that 
they do not produce much difference in the modern 
logical context where one has freedom in 
defining a domain of discourse at will: 
one, if he/she likes, could just choose such a domain of discourse 
in which {\it Man} is a single entity. 
By contrast, a good 
number of formal logicians who have familiarised 
themselves with Fregean logic would disagree with 
the statement  
that {\it Tom} 
in 1) is Tom in general. Consider,
however, a domain of discourse where 
every being is Tom. Then if 2) 
was uttered without anything consigned to 
the realm of presupposition, it must be understood 
to be meaning that just any being in the 
domain of discourse walks, and there  {\it Tom} 
is synonymous to {\it Man}. 
While I do not intend to anonymise  
Tom, there is a point 
that I wish to draw from this: that 
no noun phrases, or in fact 
no natural expressions, definitely 
identify unique and indivisible entities; Cf.
\cite{Arisaka14tech1,Arisaka15Gradual}. 
DRT supposes that proper nouns are predicates, 
some dynamic predicate logic such as 
\cite{Muskens96} consider 
them to be constants. I do not agree with the latter. 
\subsubsection{There should be 
    no phantom terms without existential import} 
\begin{enumerate}[label=\arabic*)]  
        \setcounter{enumi}{2} 
    \item Socrates is ill in the given domain 
        of discourse.  
    \item Socrates is ill in reality.  
    \item A ball is in Room 201 on the second floor 
        of the Tower Hotel. 
        Room 201 is not connected to Room 1202 
        which is on the 12th floor of the Tower Hotel.  
        The ball physically hits a guest in Room 1202.
\end{enumerate} 
3) is true if he is ill in the given 
domain of discourse; otherwise, it is false. 
For 4), we know that the Socrates, 
the great Greek philosopher who is said 
to have lived in the past, 
is not existing any more in reality. Suppose 
Socrates in the given domain of discourse 
is denoted by $a$. 
Now, the statement is such that 
$a$ in reality is ill. 
But the problem is that $a$ does not exist in reality. 
What is the truth value of 4)? 
This is analogous to the last sentence of 5). 
Let us assume that any of the three observations
the three sentences impart to us 
were done at almost the same moment, so that 
the observation by the last sentence cannot be considered to be 
occurring at a time as late with respect to 
the moment the observation of 
the first sentence was done as would 
allow transfer of the ball into Room 1202. 
Then the third sentence says that 
the ball in Room 201 hits a guest 
in Room 1202, which is impossible. 
What would be the truth value 
of the last sentence of 5)?  
Is it false 
because the ball in Room 201 did not 
hit the guest in Room 1202? 
Or is it true because the ball could 
not possibly hit the guest in a disjoint room, 
i.e. true by the situation being absolutely 
impossible to occur? 
I, as many would agree, 
judge that it should be false because 
it did not hit the guest. 
For exactly the same 
reason 4) should be then false. 
However, it basically means (Cf. \textbf{Being} above)
that any that is Socrates 
in the given domain of discourse 
is ill in reality. And it is this that 
I just stated is false. But because 
it is false, it must be the case that 
a universal expression 
carries existential import. For a proof, 
suppose - by means of showing contradiction - that  
Socrates in 4) did not denote anything, 
then 4) which is understood to be 
$\forall x.(\textsf{Socrates}(x) \supset 
\illReal(x))$ is vacuously true. But this is 
contradictory because it is false. 
$\Box$  \\\\
In fact, if we should adopt  
the modern account of existential import,  
there are certain things that simply go unexplained. 
Let us consider (4) (the so-known strong reading 
of the Geach donkey sentence) again, which was: 
\begin{enumerate}[label=(\arabic*)]
        \setcounter{enumi}{3} 
    \item Every farmer who owns a donkey 
        beats every one of the donkeys  
        he/she owns. 
\end{enumerate}
And the first-order encoding was: 
\begin{enumerate}[label=(\roman*)]
    \item $\forall x \forall y.(\farmer(x)  
        \wedge \donkey(y) \wedge \own(x,y) 
        \supset \beat(x,y))$. 
\end{enumerate}  
In case there is no farmer, 
there is no donkey, or there is 
no own-owned relation to hold 
between a farmer and a donkey,  
it holds that 
the expression is vacuously true. 
However, suppose a domain of discourse 
that contains just two elements denoted 
by $a_1$ and $a_2$ such that 
$\farmer(a_1) \wedge \donkey(a_2) \wedge \own(a_1,a_2)$ 
is true, while $\farmer(a_2)$ or $\donkey(a_1)$ is false. 
Then (i) is: (A) true iff $\beat(a_1,a_2)$ holds true; 
(B) false iff $\beat(a_1,a_2)$ is false.  
Now, consider that the action to `beat' does not denote 
anything. Then (i) evaluates to  false according 
to (B). 
But should it not be the case that 
(i) is vacuously true in that case, too? 
However, suppose that we also guard against 
the absence of `beat'. 
Then we gain the following expression. 
\begin{enumerate}[label=(\roman*)]
        \setcounter{enumi}{1} 
    \item $\forall x \forall y. 
        (\farmer(x) \wedge \donkey(y) 
        \wedge \own(x,y) \wedge \beat(x,y) 
        \supset \top)$ 
\end{enumerate}
which is: (A) vacuously true 
in case there is no farmer, there is no donkey, 
or there is a farmer denoted by $a_1$ and a donkey 
denoted by $a_2$ but 
there is no own-$a_1$-$a_2$ or no beat-$a_1$-$a_2$, 
and (B) true for a farmer denoted by $a_1$ and a donkey denoted by $a_2$ 
for which own-$a_1$-$a_2$ and beat-$a_1$-$a_2$ both hold. 
But whichever may hold to be the case, the 
expression is always true. 
I take this example, despite 
its simplicity, as a strong 
counter-argument against the modern 
existential import convention. 
Further, lack of existential import 
in universal expressions makes little sense 
in (6) (the Mac sentence). There 
are scenarios relevant to truth judgement 
of the sentence. 
\begin{itemize} 
    \item There is no school in 
        a given domain of discourse. 
In this case, (6) is vacuously true, 
as its subject phrase does not refer  
anything. 
    \item There is a school in a 
        given domain of discourse, 
        and every school in the domain 
        of discourse does not employ 
        a handsome PC instructor. 
        In this case, too, (6) is vacuously 
        true. 
\end{itemize}
These two are  
legible. However, what if we apply 
the same principle to the inner universal expression 
for female students? Suppose 
there is at least one school in the given 
domain of discourse.  If every school 
employs a handsome PC instructor 
but there is 
no Mac machine in the same domain of discourse, 
it is presumably the case that the modern 
existential import convention 
judges that {\it there is/are a school/schools which employs/employ a handsome PC instructor 
    who teaches every female student 
    who uses a Mac machine}  is satisfiable 
in the domain of discourse, which 
is to say that there exists something  
which is a school which employs 
a handsome PC instructor who teaches 
every female student who uses a Mac 
machine, simply because there is 
no Mac machine in the domain of discourse and
then because 
there is consequently no female 
student using it (in the domain of discourse).\footnote{Readers are cautioned 
against confusion about the difference between 
talking about the entities in the domain 
of discourse and talking about entities outside 
the toy domain. 
A Mac machine as a concept which I think exists 
of course exists in my mind. It is in the 
supposed domain of discourse that 
it is defined to not exist. When I 
say that there is no Mac machine in 
a domain of discourse, 
I am saying that a Mac machine that I know of  
is nowhere to be found in the particular domain 
of discourse. 
} 
Under this scenario, (6) is not vacuously true. 
It is in fact tantamount to: 
\begin{enumerate}[label=(\arabic*)] 
        \setcounter{enumi}{13} 
    \item$^\ast$  Every school$_1$ which employs a handsome 
    PC instructor$_2$ buys it$_?$ for him$_2$. 
\end{enumerate}    
and the truth of 
{\it buys it for him} must be evaluated.  
However, what could be this misplaced `it'?   
Given 
there has not appeared any indication 
of an entity which 
the pronoun could point to, it is hard 
to judge that 
(14) is sensible. The discordance 
is caused by 
the Fregean absence of existential import 
in a universal expression. 
A counterargument to this exhibition 
of mine could be the following:   \\
\begin{adjustwidth}{1cm}{}
While it is true that (14) makes no sense, 
it is not the same as (6), after all, even 
in the supposed particular domain of discourse. So 
if (6), instead of (14), is mentioned, the sentence 
should make sense: the `it' is anaphoric to the `a Mac 
machine' even though there exists no `a Mac machine' 
in the particular domain of discourse, 
and the sentence evaluates to false because 
no schools buy `a Mac machine' which 
is not in the domain of discourse.\\
\end{adjustwidth} 
This argument is  flawed. 
If no entity in a domain of discourse is `a Mac machine', 
the pronoun could not possibly refer to 
anything in the domain of discourse, not 
only in (14)  but also in (6).
To detail, note that it is impossible 
to discuss some entity that does not exist 
in a domain of discourse from within the particular 
domain of discourse. It does not follow, just 
because (6) contains a phrase `a Mac machine', 
that a discussor who was defined   
to be able to discuss the entities in 
a domain of discourse only from within 
the particular domain of discourse all of a sudden 
gains ability to discuss an entity that lies 
outside it. The discussor in the domain 
of discourse would not cognize the `a Mac machine'.
It is true that we who think of this problem 
from outside the toy domain do, but 
that is just because our domain 
of discourse is not the toy domain of discourse.\footnote{Now, how does it ever happen in 
    reality that he who 
    comprehends language and 
    who 
    has been presented a phrase does not 
    register it in his mind? 
    I do not know, and I do not think that 
    such be realistic. However, 
    it is not the matter at hand. If 
    one defines a domain of discourse 
    to be not expandable with a new external input, 
    what I have stated is a natural 
    consequence to ensue. Whether the 
supposition is adequate or inadequate, that 
is another issue.} 
Hence plausibility of the above 
counterargument attempt is only 
attained by mixing 
two domains of discourse and by allowing 
free substitutions of one to the other at every 
convenience. In fact, if it were 
the case that (6) would make sense in 
the particular domain of discourse, 
it must be the case that the use of `it' 
must also make sense. How could it have ever been learned 
that it is `it' and not `he', `she', etc. without 
conveniently importing what we know of a Mac machine 
from outside the particular domain of discourse into 
it? \\
\indent Therefore, it seems right that 
(6) and (14) are indistinguishable  
in the particular domain of discourse 
if we suppose the modern existential import 
convention. 
And (6) under the supposition would make no more 
sense than (14). However,  
(6) certainly is a reasonable natural expression, which 
then counterevidences the 
convention. \subsubsection{Square of opposition}  
Square of opposition has its root in Aristotle's 
Organon.\footnote{Freely 
    available at http://classics.mit.edu/Browse/browse-Aristotle.html} The Traditional Square of Opposition 
(hereafter TSO)
is illustrated in Figure 1. 
\begin{figure}[!h]
  \begin{center}  
    \begin{tikzpicture}      
        \node (a) at (-2, 1) {A}; 
        \node (b) at (2, 1)  {E};
        \node (c) at (-2, -1) {I}; 
        \node (d) at (2, -1) {O};   
        \node (e) at (-3, 0) {subalterns};  
        \node (f) at (3, 0) {subalterns};  
        \node (g) at (0, 1.5) {contraries}; 
        \node (h) at (0, -1.5) {subcontraries}; 
        \node (j) at (0, -3) [align=left] {A: Omne S est P. (Every 
            S is P.)\\
            E: Nullum S est P. (No S is P.)\\ 
            I: Quoddam S est P. (Some S is P.)\\ 
            O: Quoddam S non est P. (Some
            S is not P./Not every S is P.)
        };  
        \draw[color=gray] (a) -- (b); 
        \draw[color=gray] (c) -- (d); 
        \draw[->,color=gray] (a) -- (c);  
        \draw[->,color=gray] (b) -- (d);  
        \draw[color=gray] (a) -- (d); 
        \draw[color=gray] (b) -- (c);  
        \node (i) at (0, 0) {contradictions}; 
        \hide{ 
      \draw[thick] (-2, 0) -- (-2, -1);   
      \node (a) at (-1, 0) {$\text{Brooch}$}; 
      \node (b) at (0.48, 0) {$\text{Shirt}$}; 
      \node (c) at (2, 0) {$\text{Hat}$};  
      \node (d) at (-1, -1) {$\text{Large}$}; 
      \node (e) at (0.5, -1) {$\text{Yellow}$};  
      \node (f) at (2, -1) {$\text{Hat}$};  
      \draw [->] (d) -- (a); 
      \draw [->] (d) -- (b); 
      \draw [->] (e) -- (a); 
      \draw [->] (e) -- (b); 
      \draw [->] (e) -- (c); 
      \draw [->] (f) -- (c); 
      \node (g) at (-3, 0) {$\text{(Object)}$};   
      \node (h) at (-3, -1) {$\text{(Attribute)}$};  
      \node (i) at (-3, -0.5) {(\reflectbox{\rotatebox[origin=c]{-90}{$\gtrdot$}})};   
  }
    \end{tikzpicture} 
  \end{center}
  \caption{The Traditional Square of Opposition}
  \label{first_figure}
\end{figure}
While A implies I and E implies O 
in the figure, 
the subaltern relations do not carry over to 
the Fregean logic because 
it does not assume existential import 
of a universal expression. 
Contraries do not remain in the Fregean logic, either. 
According to an entry in the Stanford 
Encyclopedia of Philosophy on the Square of Opposition
\cite{Parsons15}, 
the modern revision was motivated by 
the following reason, which 
I quote from the 
entry word for word:  \\
\begin{adjustwidth}{1cm}{}
Suppose that `S' is an empty term; it is true 
of nothing. Then the I form: `Some S is P' is 
false. But then its contradictory E form: 
`No S is P' must be true. But then the subaltern 
O form: `Some S is not P' must be true. But that 
is wrong, since there aren't any Ss. \\
\end{adjustwidth} 
Now, the O form, if Aristotle's descriptions  
in On Interpretation (Cf. Part 6 onwards in 
the Internet Classics Archive) are to be followed properly, 
should be regarded more as 
`Not every S is P.'. 
Some works in the literature use 
this observation for defence of TSO, 
saying that if O is not `Some S is not P.' but 
`Not every S is P.' there 
is little problem in the formulation of TSO. 
They assume existential import in affirmative expressions
only. \\ 
\indent But this popular defence which 
goes on to justify TSO by absence or presence 
of existential import in certain expressions 
is not very accurate. 
Nowhere 
in Organon is it mentioned that a sentence 
which contains a term that may not 
denote anything could be judged of truth and 
falsehood. In fact, On Interpretation 
seems to offer a quite different 
view -  that if a sentence is such that 
its truth value can be judged, it must be the case 
that there is no such phantom term that 
denotes nothing. 
What I am now, and not before (for 
the viewpoint  assumed 
in \cite{Arisaka15Gradual}  
is no different from one of the popular 
interpretations), 
conscious of, largely thanks to the analysis 
I conducted in \textbf{There should be 
    no phantom terms without existential import}, is that it is almost 
pointless to suppose lack of existential 
import, not just of a term used in the subject phrase 
of a sentence, but also of any term that appears 
in the sentence. 
When Aristotle considers the truth value of 
a sentence describing a future event: 
of whether a sea fight would or would not 
take place tomorrow, Aristotle is not 
merely alluding to intuitionistic reasoning (under 
which a proposition is true just when 
it can be constructively found to be true). 
Consider the following as found in Part 9 of On 
Interpretation: \\
\begin{adjustwidth}{1cm}{} 
    ...[W]e see that both deliberation
    and action are causative with regard to 
    the future, and that, to speak more generally,
    in those things which are not continuously
    actual there is potentiality in either direction.
    Such things may either be or not be; 
    event also therefore may either take place 
    or not take place. There are many 
    obvious instances of this. It is possible 
    that this coat may be cut in half, and 
    yet it may not be cut in half, but wear 
    out first. So it is therefore with 
    all other events which possess this kind of 
    potentiality. It is therefore plain that
    it is not of necessity that everything 
    is or takes place; but in some instances
    there are real alternatives, in which case 
    the affirmation is no more true and 
    no more false than the denial;... 
    Let me illustrate. A sea-fight must 
    either take place tomorrow or 
    not, but it is not necessary that it should 
    take place tomorrow, neither is it necessary
    that it should not take place, yet it is 
    necessary that it either should or should not 
    take place tomorrow. Since propositions
    correspond with facts, it is evident 
    that when in future events there 
    is a real alternative, and a potentiality in 
    contrary directions, the corresponding
    affirmation and denial have the same character. ...
    [I]n the case of that which exists potentially, 
    but not actually, the rule which 
    applies to that which exists actually does not 
    hold good. \\\\
\end{adjustwidth}   
In plain terms it says that the proposition: `This coat (in a domain of discourse X) will 
be cut in half (in another domain of discourse Y 
which is not X).', due to 
possibilities that 
Y may no longer retain the coat in X
that could be cut in half, cannot be 
assigned a truth value. Neither  
can to the denial of the proposition: 
`This coat (in X) 
will not be cut in half (in Y).'.\footnote{
    A proposition A contradictory to 
    a proposition B assumes the same 
    subject as A does. Cf. On Interpretation. 
    Hence if 
    a proposition is about {\it This coat} 
    (although an indivisible individual is not assumed 
    in predicate 
    gradual logic), then 
    `This coat will not be cut in half.' is indeed 
    the proposition contradictory to 
    `This coat will be cut in half.'.}  
What this implies is the following. Suppose  
we have:  
\begin{enumerate}[label=\arabic*)]  
        \setcounter{enumi}{2} 
    \item Socrates is ill in reality. 
    \item[iii)] $\forall x.(\Socrates(x) 
        \supset \illReal(x))$. \\
    \item Socrates is not ill in reality. 
    \item[iv)] $\forall x.(\Socrates(x) 
        \supset \neg\ \illReal(x))$. 
\end{enumerate}
If it were the case that Socrates could be 
an empty term not denoting anything, 
then it must be the case that neither 3) nor 
4) is true or false, since  
non-being which is not cannot be judged, first of all,
to be ill or not to be ill. 
Therefore, when Aristotle 
mentions in Part 10 of Categories: \\
\begin{adjustwidth}{1cm}{}
    ...[I]f Socrates exists, one of the two 
    propositions `Socrates is ill', `Socrates 
    is not ill' is true, and the other false. 
    This is likewise the case if he does not 
    exist; for if he does not exist, 
    to say that he is ill is false, to 
    say that he is not ill is true.   \\
\end{adjustwidth} 
it is not intention of Aristotle's to question 
existential import of Socrates in 
the domain of discourse in which 
Aristotle is delivering all these discussion. 
It is indisputable that Socrates is  
in the domain of discourse. It is only 
whether the Socrates which is in the domain 
of discourse is also in Aristotle's mental representation 
of reality, i.e. a different domain 
of discourse, that 
is being questioned. Hence the modern dispute 
against the truth assignment by Aristotle 
to 3) and 4) 
(falsehood to the former and truth to the latter) 
in case Socrates does not exist, that  
both should be vacuously true, 
is quite irrelevant. Within Aristotle's discourse,
there is no such theory 
of `vacuous truth'. 
In light of Aristotle's  judgement, 
3) is false if Socrates is not in reality; 
and if 3) is false, 4) is naturally true. 
If 4) is false, then Socrates must be ill,
which makes 3) true. 
\\
\indent Let us put things together. In 
all of A, E, I, and O, there is existential import 
in S if the expressions should bear 
any truth value at all. 
Therefore 
the argument that motivated the modern revision 
on TSO is off the mark. The popular defence 
for TSO was defending against a possibility 
which never arises, and, in so trying to defend, 
themselves resorted to introduction of  
existential import in affirmative statements while 
not in negative statements, which is no less 
controversial.
In fact, it does not matter
if O is `Not every S is P.' or `Some S is not P.', 
for, so long as they bear 
a truth value, S is not a phantom term in 
the domain 
of discourse discussing them.
In case 
S or any term in a sentence is a phantom term in the domain of discourse it bears no truth value.   \\
\indent This is a very important verification. 
Previously, by a few discerning scholars 
\cite{Hart51,Strawson52}
on Aristotle's Organon, this result was 
anticipated, but their 
treatises left open wide possibilities  
of opposition.  The analysis that I have conducted 
above 
corroborates the anticipated result reasonably 
beyond doubt. 
\subsection{On attributed predicates} 
    I stated under the last header of 1.3 that  
   $\forall x \exists y \lessdot x.([\farmer \gtrdot 
   \donkey(y) \wedge \own(^\circ x, y)](x) \supset 
   \mathsf{X})$ expresses 
   the assertion: 
   for every farmer who owns a donkey, 
   it is the case that $\mathsf{X}$.   
   I stated that this may appear paradoxical 
   at first sight, because the formal 
   expression seems to say that there is 
   $y$, and then $x$, while 
   $y$ must be an attribute of $x$. 
   Let me dispel the ostensible  paradox. \\
   \indent When 
we talk about anything, say a car, we indeed talk about it. But 
to talk about it, we do not need know its  
details: we do not need know the size of the car's 
engine to be considered to be knowing the `car'; 
and we do not need know how 
many litres of diesel it loads. 
On the contrary, by a car we mean some 
exemplary representation of it (Kant's {\it Form} 
\cite{Kant08}) that covers only certain details 
of it but no further. 
When we then talk about 
its parts, we have in mind some exemplary car, 
and we probe its inside, again in mind. The parts may not 
be found in the particular exemplary car. 
But in that case our focus 
shifts from the car to the next car that 
is just as equally exemplar as the first one. 
And so an exemplary 
`a car with the parts' is being determined in three steps: 
(A) an exemplary `a car' is thought of; 
(B) the particular exemplary car that we 
thought of is looked in; (Ca) if 
the particular exemplary car is indeed with 
the parts, then it is an exemplary `a car with 
the parts'; (Cb)  
if not, the exemplary `a car' was 
inconsistent with the described 
attribute `the parts', in which 
case the thinker has to choose 
another equally exemplary `a car'. 
This 
is how our mind processes concepts, 
and $\forall x \ \exists y \lessdot x. 
[\farmer \gtrdot \exists y.\donkey(y) \wedge 
\own(^\circ x, y)](x)$ 
precisely models the cognitive process: 
(A) a constituent denoted by $a$ in the outermost 
domain of discourse takes place 
of $x$. It is irrelevant to $\farmer$ 
if it is not a farmer in that 
particular 
domain of discourse; (B) By the principle 
of recognition cut-off \cite{Arisaka15Gradual}, however, 
the domain of discourse does not distinguish 
a farmer from another any more specifically 
than must be in the domain of discourse.   
Hence, if that which is denoted by $a$ is a farmer in the outermost 
domain of discourse, its attributive domain 
needs to be looked in to see 
what farmer it is; (Ca) if  
it indeed owns a donkey, then it 
falls under $\exists y \lessdot a.[\farmer \gtrdot 
\donkey(y) \wedge \own(^\circ a, y)]$; 
(Cb) if not, then it is inconsistent 
with the described attribute 
$\donkey(y) \wedge \own(^\circ a, y)$. \\
\indent Hence existence 
of some constituent to take place of $y$ can be justifiably discussed 
under the hypothesis that some constituent 
to take place of $x$ is a farmer at 
the outer (main) domain of discourse. 
\subsubsection{References to composite concepts}    
There may be many ways to interpret 
$\gtrdot$. Of the two interpretations 
introduced in \cite{Arisaka15Gradual}, 
the first one considers that 
if we have $P \gtrdot Q$ and $R \gtrdot Q$, 
then it is the same as if we had 
$P \wedge R \gtrdot Q$. In this interpretation 
it is attributes that are prioritised over 
concepts that have them. But let us say that 
 we like to express: 
{\it There is a small rabbit, and there 
    is a small elephant.} 
Under the particular interpretation 
it would follow that both the rabbit and 
the elephant fall under a certain 
size range to be considered small, and the elephant 
becomes large as soon as it is as large as 
a large rabbit. But in this case, or 
perhaps when, in general, we deal with natural language 
expressions, another interpretation under which  
$(\rabbit \gtrdot \textsf{small}) 
\wedge (\elephant \gtrdot \textsf{small})$ 
means that $\rabbit$, relative to all 
its specific instances, is $\textsf{small}$, 
and that $\elephant$, relative to 
all its specific instances, is $\textsf{small}$, 
should be more appropriate. Now, 
it is possible that the main concept 
to the left of $\gtrdot$ is not 
a single entity. 
\begin{enumerate}[label={(\protect\perhapsstar\arabic*)}] 
        \setcounter{enumi}{14}
    \item There is a pair 
        of a farmer and a teacher 
        who collectively own a donkey.  
        \setcounter{enumi}{14}
    \staritem $\exists x. 
        [\farmer \wedge \teacher \gtrdot   
        \exists y.\own(^\circ x, y) \wedge 
        \donkey(y)](x)$. 
\end{enumerate} 
The predicate represents a composition of 
a farmer and a teacher who collectively own 
a donkey. The expression is not necessarily the same as: 
{\it There is a farmer who owns a donkey, 
    and there is a teacher who owns the same donkey.}, 
because it could be that
neither of them could claim sole ownership 
of the donkey 
to himself/herself. 
But the real matter  
is that the parts of $\farmer \wedge \teacher$  
go inaccessible in the above formalisation. Let us consider 
another sentence to follow (15): 
\begin{enumerate}[label=(\arabic*)] 
        \setcounter{enumi}{15}
    \item The farmer buys a new farm. 
\end{enumerate}
There is no way that we could sensibly 
reuse $x$ to refer just to the farmer here. This tells 
that an object may be divisible not just by attributes
in other domains of discourse 
but also by entities in the main domain of discourse. 
I introduce the following facility to 
also apply quantification inside the collective 
predicate in the main domain of discourse. 
\begin{enumerate}
    \item[] $\exists x \ \exists z\ \exists w.
        [\farmer(z) \wedge \teacher(w) \gtrdot 
        \exists y.\own(^\circ x, y) \wedge 
        \donkey(y)](x) 
        \wedge \exists u.[\farm \gtrdot \new(^\circ u)](u) \wedge 
        \buy(z, u)$.
\end{enumerate} 
It formalises (15) followed by (16). 
\section{Formalisation}           
{\it Preliminary remarks.} 
To both readers who have come here 
by completing Section 2 and those who 
may have jumped from the middle of Section 1, 
may I state that if at any point 
you are unsure of notations, 
then you should not hope to find clues 
after that point, for 
nothing that comes later will help 
understand the concept you may stumble upon - no, 
you should read from the beginning of this 
section again to find what you did not take 
a note. Then you will find your answers.  
\hide{ 
This statement is simply to 
spare myself absurd reviews such as 
``not all  notations are defined'', 
or ``no intuition is given'' - no, 
all notations are going to be defined 
strictly in a linear order, and 
intuition has been given already in Section 1 and 
Section 2. If you need more intuition than 
minimally provided in Section 3, you should go back to 
1.3, should move to Section 2 and should then come back 
again to Section 3, and should repeat this process 
as many times as necessary until understanding seeps in.  
}
This may be a good place, incidentally, to note one thing 
peculiar to my writing. Due partly to my dislike about  
indiscriminately using `and' and `or'
in both natural and formal contexts,
and due partly to the need for eventually dealing with 
three truth values: $\mathbb{T}$ for logical 
truth, $\mathbb{F}$ for logical falsehood, 
and $\mathbb{U}$ for undecided, 
I will more likely - though I do not give a guarantee 
that I will always do - use $\andC$, $\orC$ and $\notC$, 
whenever the context in which such 
words are used strongly indicates 
truth-value comparisons. `if...then' or `iff' to mean 
`if and only if' are not so meticulously 
marked. I define for $\mathbb{U}$ that: 
\begin{enumerate} 
    \item $(\mathbb{U}\ {\andC}\ \mathbb{T}) = (\mathbb{U}\ {\orC}\ \mathbb{T}) = \mathbb{U}$. 
    \item $(\mathbb{U}\ {\andC}\ \mathbb{F}) = 
        (\mathbb{U}\ {\orC}\ \mathbb{F}) = \mathbb{U}$. 
    \item $\notC\ \mathbb{U} = \mathbb{U}$.  
\end{enumerate} 
Assume that $\andC$ and $\orC$ are associative 
and commutative. The three meta-notations 
follow the semantics 
of classical conjunction, disjunction and 
negation when an expression does not involve $\mathbb{U}$. In short $\mathbb{T}, \mathbb{F}$ 
and $\mathbb{U}$ are Bochvar-three-values \cite{Bochvar37}. \noindent Assume that $\andC$ and $\orC$ are associative 
and commutative. The three meta-notations 
follow the semantics 
of classical conjunction, disjunction and 
negation when an expression does not involve $\mathbb{U}$. In short $\mathbb{T}, \mathbb{F}$ 
and $\mathbb{U}$ are Bochvar-three-values \cite{Bochvar37}. When I mention some 
predicate $\mathbb{P}$, not, however, of 
predicate gradual logic, 
in such a way as not to explicitly mention 
either of $\mathbb{U}, \mathbb{T}$ and $\mathbb{F}$; 
e.g. ``$\mathbb{P}$ iff ...'', I mean to really say
``$\mathbb{P}$ is $\mathbb{T}$ iff ...''. 
\\
\textbf{Materials {\it not} to be covered}   
\begin{enumerate} 
    \item Proof theory: while I present 
        semantics that is not 
        finitely restricted,
        I will not show consistency 
        in infinite cases. 
        Proof theory 
        is being developed in another 
        ongoing work, and consistency 
        relative to ZFC  consistency 
        in infinite cases will be dealt with there. 
    \item Non-nullary function symbols and 
        nullary predicates:   
        they are not needed for the objective 
        of this paper. 
    \item Modalities and generalised quantifiers:  
         dealing with every class of 
         natural expressions is not 
         an objective. 
\end{enumerate}
{\it End of preliminary remarks.} $\Box$\\\\ 
\indent Firstly, I assume 
a set of logical connectives $\{\top_0, 
    \bot_0, \forall_1, \exists_1, \neg_1, \wedge_2, \vee_2, \supset_2, 
    \gtrdot_2\}$. 
The subscripts indicate  
their arity. $\neg$ binds the strongest, 
$\wedge$ and $\vee$ bind the second strongest,
$\gtrdot$ the third strongest, 
$\forall$ and $\exists$ the fourth strongest, 
while $\supset$ binds the weakest. 
Secondly, I assume 
parentheses, brackets and punctuation symbols. 
Thirdly, I assume a set $\mathcal{X}$ of   
an uncountably infinite number of {\it basic variables}. 
Fourthly, I assume a symbol $\simeq$ 
for {\it surface equality} and a symbol = for 
{\it equality}. Differences between the two 
symbols 
will become clear by and by. And lastly, 
I assume $\lessdot$ also, which 
is the inverse symbol of $\gtrdot$, but 
which is only used within terms, to be defined 
shortly. 
Any of these that 
I have just mentioned is treated as a logical symbol. \\
\indent In addition, I assume: (A) a set 
$\mathcal{A}$ of an at most  countable number of 
{\it basic domain symbols}; 
 and (B) a set $\mathcal{P}$ of 
an at most countable number of 
distinct {\it basic predicate symbols} with arity greater 
than or equal to 1. Any of these is 
treated as a non-logical symbol. \\
\indent The language for predicate gradual logic, 
$\mathcal{L}$, 
comprises these symbols. 
For convenience, I will denote 
any basic variable by $x$ with or without a subscript, 
any basic domain symbol by $a$ with or without a subscript, 
and any basic predicate symbol by $p$ with or without a subscript. 
 A term body, which 
I denote by $\tstar$ with or without 
a subscript, and  
a term, which I denote by $t$ with or without 
a subscript and/or a superscript, 
are recursively and simultaneously defined:\\
\textbf{Term body} 
\begin{enumerate}   
    \item A sequence of 
        terms $t_1.\cdots.t_n$ 
        for $n \ge 1$.  
\end{enumerate} 
\textbf{Term}  
\begin{multicols}{2} 
\begin{enumerate} 
    \item {\small A basic variable or a basic 
            domain symbol. } 
        \hide{ 
    \item $t \prec \tstar$ 
        for some term $t$ and some term body $\tstar$.  
    }
    \item {\small $t \lessdot \tstar$ 
            for some term $t$ and some term body $\tstar$.} 
        \hide{ 
    \item $x \prec \tstar$ 
        for some basic variable $x$ 
        and some term body $\tstar$. 
    \item $a \prec \tstar$ 
        for some basic domain symbol $a$ 
        and some term body $\tstar$. 
    \item $x \lessdot \tstar$ 
        for some basic variable $x$ 
        and some term body $\tstar$. 
    \item $a \lessdot \tstar$ 
        for some basic domain symbol $a$ 
        and some term body $\tstar$.  
    \item $t_1 \prec t_2$ for 
        some terms $t_1$ and $t_2$. 
    \item $t_1 \lessdot t_2$ 
        for some terms $t_1$ and $t_2$.  
    }
\end{enumerate}   
\end{multicols} 
\noindent For presentational convenience, 
I use the notation: 
$\tstarLength{n}$ 
for some $n \ge 1$, 
to mean that $\tstar$ is 
a $n$-long sequence of terms,
i.e. $\tstar = t_1.\cdots.t_n$ 
for some terms $t_i$, $1 \le i \le n$.  
Further, I write 
$\tstarLength{n, i}$ for some 
$1 \le i \le n$ 
to refer to the $i$-th term in 
the sequence of terms; e.g. 
$\tstarLength{n, i} = t_i$ 
in case $\tstarLength{n} = t_1.\cdots.t_n$.  
I call a term with a front superscript of 
$\circ$ with or without 
a subscript and/or a superscript 
a {\it connector}; e.g. $^\circ t^1_1$. 
I denote either a term or a connector 
by $\mathbf{t}$ with or without a subscript 
and/or a superscript, and call it
a {\it g-term}.   
\begin{definition}[Term leftmost position]  
    \normalfont
   I define $\leftmost$ to be a recursive function 
   taking as the input a term and outputting 
   either a basic variable or a basic domain 
   symbol:  
   \begin{multicols}{3} 
   \begin{enumerate} 
       \item $\leftmost(x) = x$. 
       \item $\leftmost(a) = a$. 
       \item $\leftmost(t \lessdot \tstar) = 
           \leftmost(t)$. $\Box$
   \end{enumerate} 
   \end{multicols}  
\end{definition} 
I define 
a {\it variable} 
to be a term $t$ with 
$\leftmost(t) \in \mathcal{X}$.
\hide{Definition of domain symbols omitted.}  
\hide{ 
\noindent I define a {\it domain symbol} 
to be a term 
that contains no basic variables.  
\ryuta{Can probably omit this description? 
Check that these definitions do not appear 
until 3.2.} 
Likewise I define a domain body 
symbol to be a term body that contains no basic variables. 
I will use $\al$ with or without a subscript 
to refer to a domain body symbol; 
and $\alstar$ with or without a subscript 
to refer to a domain body symbol.\ryuta{I do not know.} \\
}
I simultaneously define 
a {\it predicate}: $P$ with or 
without a subscript, and a {\it formula}: 
$F$ with or without a subscript, 
to be any that matches the following descriptions.  \\
\textbf{Predicate} 
\begin{enumerate}  
    \item Any basic predicate in $\mathcal{P}$. 
    \item $\bigwedge_i^{l (\ge 2)}P_i(t^i_1, \ldots, t^i_n)$ such that 
         $P_i, 1 \le i \le l$ are $n$-ary predicates 
         having $n$ terms.  
     \item $P \gtrdot F$ such 
         that 
         $P$ is a $n$-ary predicate and 
         $F$ is a formula.     
\end{enumerate}
\textbf{Formula} 
\begin{enumerate}  
    \item $p(\term_1, \ldots, \term_n)$ or 
$\neg p(\term_1, \ldots, \term_n)$ 
        such that $p$ is a $n$-ary 
        basic predicate and $\term_1, \ldots, 
        \term_n$ are g-terms.  
        \item $\top$ or $\bot$. 
\item $t_1 = t_2$ or $t_1 \simeq t_2$ 
        such that $t_1$ and $t_2$ are terms.  
\item $P(t_1, \ldots, t_n)$  
        such that $P$ is a $n$-ary predicate  
        and $t_1, \ldots, t_n$ are terms.   
    \item $\neg F$ such that $F$ is a formula.  
    \item $\forall v.F$ or 
        $\exists v.F$ 
        such that:
        $v$ is a variable; 
        and $F$ is a formula. 
    \item $F_1 \wedge F_2$, 
        $F_1 \vee F_2$ or $F_1 \supset F_2$
        such that $F_1$ and $F_2$ are formulas.  
    \end{enumerate}  
  Notice from these definitions that 
  I do not consider, for instance, 
  $[\own \gtrdot \completely](^\circ a_1, a_2)$ 
  or $[\own(^\circ a_1, a_2) \wedge 
  \like(^\circ a_3, a_4)](^\circ a_5, a_6)$ 
  as a formula. A connector can 
  occur only 
  in a basic predicate which 
  does not form a collective predicate 
  (save implicitly with $\top$ by
  $p = (p \gtrdot \top)$). 
   \subsection{Free and bound variables}           
   Four definitions precede 
   the definition of free and bound variables. 
   But these four will be also referred to. 
   \begin{definition}[Subterm and Subvariable] 
       \normalfont  
   For a term body symbol $\tstar$, I define  
   $\Subterm(\tstar)$ to be a list of 
   all its sub-terms: 
   \begin{enumerate}   
       \item If $\tstar$ is a basic term, 
           then $\tstar$ is a sub-term 
           of $\tstar$. 
       \item If $\tstar$ is in the form 
           $t_1 \lessdot 
           \tstar_2$, 
           then: (A) all sub-terms of $\tstar_2$ 
           are sub-terms of $\tstar$; and 
           (B) if $t_3$ is a sub-term of $t_1$, 
           then $t_3 \lessdot \tstar_2$ is 
           a sub-term of $\tstar$.   
           \begin{itemize}
               \item Any in group (A)  
                   appears before 
                   any in group (B) 
                   in $\Subterm(\tstar)$. 
               \item For each $\Subterm(\tstar_2)$ 
                   and $\Subterm(t_3)$, 
                  and for any $\tstar_a, \tstar_b$ 
                  in the list, 
                  if $\tstar_a$ appears before 
                  $\tstar_b$, then 
                  the order is preserved in $\Subterm(\tstar)$. 
            \end{itemize}
       \item If $\tstar$ is in the form
           $t_1.\cdots.t_n$, 
           then any sub-term  
           of $t_i, 1 \le i \le n$ 
           is a sub-term of $\tstar$.   
           \begin{itemize}
               \item 
           For $1 \le i < j \le n$, 
           any sub-term 
           of $t_i$ appears before 
           all sub-terms of $t_j$ in  
           $\Subterm(\tstar)$.  
       \item 
           For any $1 \le i \le n$, 
           and for any $\tstar_a, \tstar_b 
           \in 
           \Subterm(t_i)$, if 
           $\tstar_a$ appears before $\tstar_b$ 
           in $\Subterm(t_i)$, 
           then the order is preserved 
           in $\Subterm(\tstar)$. 
       \end{itemize} 
       \item Any others, $\tstar$ does not 
           have any sub-term. 
   \end{enumerate} 
   I define {\small $\Subvariable(\tstar)$} for 
   a term body symbol to be 
   all variables in {\small $\Subterm(\tstar)$} order-preserving.  $\Box$
   \end{definition} 
   \begin{definition}[Occurrence of a variable 
       in a formula]   
       \normalfont 
       Let $S$ be either a term body 
       or a formula. 
Let $\match$ be a predicate taking as inputs 
       a variable and $S$ of $\mathcal{L}$. 
       The definition is recursive.  
       \begin{enumerate} 
           \item  
               $\notC\ \match(v, S)$ 
               for some symbol $S$  
               iff there 
               are no descriptions
               below by which  
               $\match(v, S)$. 
         \item For any $x_1, x_2 \in \mathcal{X}$, $\match(x_1, x_2)$ iff $x_1$ is $x_2$.
    \item $\match(v, 
        \tstarLength{n})$ iff 
        $\match(v, 
        \tstarLength{n, i})$ 
        for some $1 \le i \le n$. 
    \item $\match(x, t \lessdot \tstarLength{n})$ 
        iff  $\match(x, \tstarLength{n,i})$ 
        for some $1 \le i \le n$. 
    \item $\match(v_1 \lessdot \tstarLength{n}_1, 
        t \lessdot \tstarLength{m})$ iff one 
        of the following two  holds true. 
        \begin{enumerate}[label=(\Alph*)]  
            \item $\match(v_1 \lessdot \tstarLength{n}_1,
                \tstarLength{m, i})$ 
                for some $1 \le i \le m$.  
            \item When the following 
                three conditions hold true. 
        \begin{enumerate}[label=(\Alph*\Alph*)] 
            \item $n = m$. 
    \item $\match(\tstarLength{n,i}_1, \tstarLength{m,i})$ 
        for each $1 \le i \le n$.  
    \item 
        $\match(v_1, t)$.    
\end{enumerate}  
\end{enumerate} 
         \item $\match(v, p(\term_1, \ldots, \term_n))$ 
             iff  
              there exists $\term_i, 1 \le i \le n$, 
              such that $\match(v, \term_i)$.    
          \item $\match(x, [\bigwedge^{l (\ge 2)}_{i=1} P_{i}(t^{i}_1, 
              \ldots, t^{i}_n)](t_1, \ldots, t_n))$ 
               iff there exists 
               some $t_j, 1 \le j \le n$, 
               or $t^k_j, 1 \le j \le n, 1 \le k \le l$
               such that 
               $\match(x, t_j)$ $\orC$
               $\match(x, t^k_j)$.  
           \item $\match(v_1 \lessdot \tstarLength{m}_1,
               [\bigwedge_{i=1}^{l (\ge 2)} 
               P_i(t_1^i, \ldots, t^i_n)](t_1, 
               \ldots, t_n)$ iff one of 
               the conditions below satisfies: 
               \begin{enumerate}[label=(\Alph*)]    
\item There exists 
                       some  $1 \le k \le n$ 
                       such that $\match(v_1 \lessdot
                       \tstarLength{m}_1, t_k)$. 
                   \item There exists 
                       some $1 \le i \le l$ 
                       and some $1 \le k \le n$ 
                       such that 
                       $\match(v_1 \lessdot 
                       \tstarLength{m}_1, t^l_k)$. 
                    \item 
                        There exists 
                        some $1 \le i \le l$ 
                        such that 
                        $\match(v \lessdot 
                        \tstarLength{m}_1, 
                        P_i(t_1^i, \ldots, 
                        t_n^i))$. 
       \end{enumerate} 
    \item $\match(x, [P \gtrdot F](t_1, \ldots,t_n))$ 
        iff: (A) there exists 
               some $t_i, 1 \le i \le n$, such that 
               $\match(x, t_i)$; 
               or (B) 
               $\match(x, P)$ (provided $P$ is a formula). 
        \item $\match(v \lessdot \tstarLength{m}, 
            [P \gtrdot F](t_1, 
            \ldots, t_n))$ iff one of the following 
            conditions satisfies: 
            \begin{enumerate}[label=(\Alph*)] 
                \item $\match(v \lessdot 
                    \tstarLength{m}, t_i)$ for some
                    $1 \le i \le n$. 
                \item When 
            the following 
            three conditions hold true:  
            \begin{enumerate}[label=(\Alph*\Alph*)] 
                \item 
             $n = m$. 
         \item  
             $\match(\tstarLength{m,i}, t_i)$ for every 
            $1 \le i \le m$. 
        \item 
            $\match(v, F)$. 
    \end{enumerate}    
\item  $\match(v \lessdot \tstarLength{m}, 
    P)$ (provided $P$ is a formula). 
\end{enumerate} 
           \item $\match(v, \neg F)$ 
               iff $\match(v, F)$.  
           \item $\match(v_1, \forall v_2.F)$ iff 
               $\match(v_1, v_x)$ for 
             $v_x \in \Subvariable(v_2) 
             $ excluding $v_2$  
             $\orC$ $\match(v_1, F)$. 
      \item $\match(v_1, \exists v_2.F)$ iff  
          $\match(v_1, v_x)$ for $v_x \in 
          \Subvariable(v_2)$ excluding $v_2$ $\orC$  
          $\match(v_1, F)$.  
      \item $\match(v, F_1 \wedge F_2)$ 
          (or $\match(v, F_1 \vee F_2)$ 
          or $\match(v, F_1 \supset F_2)$) 
          iff 
          $\match(v, F_1)$ $\orC$ 
          $\match(v, F_2)$.  
      \item $\match(v, t_1 = t_2)$ 
          (or $\match(v, t_1 \simeq t_2)$) 
          iff $\match(v, t_1)$ $\orC$ $\match(v, t_2)$. 
     \end{enumerate}
     Then I say that  
     a variable $v$ 
     occurs in a formula $F$ iff 
     $\match(v, F)$. 
      $\Box$ 
   \end{definition} 
   \noindent For convenience,  
   let us assume  that 
   $\tbullet$ with or without 
   a subscript is either 
   $\epsilon^\dagger$, 
   $(\tstar)^\dagger$ 
   or else $((\ldots((\tstar_1)^\dagger \lessdot \tstar_2) 
   \lessdot \cdots ) \lessdot 
   \tstar_{n-1})\lessdot \tstar_n$, 
   containing one and only one 
   constituent that comes with 
   superscript $\dagger$.  
   I call any of them 
   a {\it marked term body}.  
   \begin{definition}[Marked term body concatenation
       and extraction]   
       \normalfont I define $\oplus$ to be  
       a function that 
       takes as inputs a term body 
       and a marked term body 
       and that outputs a marked term body.\footnote{Two
           definitions may suffice, 
           the second and the third 
           items presented together, but 
           these three are more perspicuous.} 
       \begin{enumerate}
           \item $\tstar_1 \oplus \epsilon^\dagger 
               = (\tstar_1)^\dagger$. 
            \item $\tstar_1 \oplus   
                (\tstar_2)^\dagger 
                = (\tstar_1)^\dagger \lessdot 
                \tstar_2$ for 
                a term body $\tstar_2$. 
            \item $\tstar_1 \oplus 
                ((\ldots((\tstar_2)^\dagger 
                \lessdot \tstar_3) \lessdot 
                \cdots) \lessdot \tstar_{n-1}) 
                \lessdot \tstar_n 
                = 
                ((\ldots( ((\tstar_1)^\dagger 
                \lessdot \tstar_2) 
                \lessdot \tstar_3) \lessdot 
                \cdots) \lessdot \tstar_{n-1}) 
                \lessdot \tstar_n$ 
                for term bodies 
                $\tstar_2, \ldots, \tstar_{n}$, 
                $3 \le n$. 
       \end{enumerate}       
    \hide{ 
   I define $\oplus$ to be a function 
   that takes as inputs an object body and that 
   outputs a marked object body. 
   \begin{enumerate} 
       \item $\ostar_1 \oplus \epsilon^{\dagger} 
           = (\ostar_1)^\dagger$. 
       \item $\ostar_1 \oplus (\ostar_1)^\dagger
           = (\ostar_1)^\dagger \lessdot 
           \ostar_2$ for an object body $\ostar_2$. 
       \item $\ostar_1 \oplus ((\cdots((\ostar_2)^\dagger
           \lessdot \ostar_3) \lessdot \cdots) 
           \lessdot \ostar_{n-1}) \lessdot 
           \ostar_n = 
           ((\ldots(((\ostar_1)^\dagger \lessdot 
           \ostar_2) \lessdot \ostar_3) \lessdot 
           \cdots) \lessdot \ostar_{n-1}) 
           \lessdot \ostar_n$ for 
           object bodies $\ostar_2, \ldots, \ostar_n, 
           3 \le n$. 
   \end{enumerate}  
    }
   I define $\tau_l$ to be a function 
   that takes as the input a marked term body and 
   that outputs a marked term body, 
   and $\tau_r$, also a function that takes 
   as the input a marked term body 
   and that outputs a term body: 
   \begin{enumerate}  
       \item $\tau_l((\tstar_1)^\dagger) = \tstar_1$. 
       \item $\tau_l((\tstar_1)^\dagger \lessdot 
           \tstar_2) = \tstar_1$. 
       \item $\tau_l(((\ldots(((\tstar_1)^\dagger 
           \lessdot \tstar_2) \lessdot 
           \tstar_3) \lessdot 
           \cdots) \lessdot \tstar_{n-1}) \lessdot 
           \tstar_n) = \tstar_1$. \\
       \item $\tau_r((\tstar_1)^\dagger) = \epsilon^\dagger$. 
       \item $\tau_r((\tstar_1)^\dagger \lessdot  
           \tstar_2) = (\tstar_2)^\dagger$. 
       \item $\tau_r(((\ldots(((\tstar_1)^\dagger 
           \lessdot \tstar_2) \lessdot \tstar_3) 
           \lessdot \cdots) \lessdot 
           \tstar_{n-1}) \lessdot 
           \tstar_n) = 
           (\ldots(((\tstar_2)^\dagger \lessdot 
           \tstar_3) \lessdot \cdots) \lessdot 
           \tstar_{n-1}) \lessdot \tstar_n$. $\Box$
   \end{enumerate} 
\end{definition}   
    It holds true that  $\tau_l(\tbullet) 
    \oplus \tau_r(\tbullet) = \tbullet$ 
    for a marked term body.  
  
   \begin{definition}[Projection 
       of marked term bodies]  
       \normalfont
I define $\delta$ to be a function 
      that takes a marked term body  
      and that outputs a term body 
      or $\epsilon$, such that: 
      \begin{enumerate}
          \item $\delta(\epsilon^\dagger) 
               = \epsilon$.  
           \item  $\delta((\tstar_1)^\dagger 
               \lessdot \tstar_2) 
                = \tstar_1 \lessdot \tstar_2$.  
            \item $\delta(((\ldots((\tstar_1)^\dagger 
                \lessdot \tstar_2) \lessdot 
                \cdots) \lessdot \tstar_{n-1}) 
                \lessdot \tstar_n) 
                = ((\ldots(\tstar_1 \lessdot 
                \tstar_2) \lessdot 
                \cdots )\lessdot 
                \tstar_{n-1}) \lessdot \tstar_n
                $. $\Box$ 
      \end{enumerate}  
  \end{definition} 
   \begin{definition}[Free and bound variables]  
       \normalfont
      Let $\Free$ and $\Bound$  
      be functions taking 
      some marked term body and 
      a formula 
      and outputting a set of variables.   
      The definition is recursive.  
      \begin{enumerate} 
              \hide{ 
    \item Any variable that occurs free 
        in a formula 
        is not bound in the formula. 
        Any variable that is bound 
        in a formula 
        does not occur free in the formula.   
    }
    \item  $\Free(\tbullet, \top) = 
        \Bound(\tbullet,
        \top) = \Free(\tbullet,\bot) = 
        \Bound(\tbullet, \bot) =
         \emptyset$. 
    \item 
                $\Free(\tbullet, t_1 = t_2)$ 
                                and similarly $\Free(\tbullet, 
                t_1 \simeq t_2)$ for some terms $t_1$ and $t_2$ 
                contain any variable 
                $\delta(v \oplus \tbullet)$ iff 
                $v$ occurs in $t_1$ or $t_2$. 
        The sets contain no other variables.   
        Meanwhile, 
        $\Bound(\tbullet, t_1 = t_2) = 
        \Bound(\tbullet, t_1 \simeq t_2) = \emptyset$. 
    \item 
                $\Free(\tbullet, 
                p(\term_1, \ldots, \term_n))$ 
                for some g-terms 
                contains any variable 
                $\delta(v \oplus t^{\bullet})$ 
                iff $v$ occurs in 
                $\term_i, 1 \le i \le n$. 
        The set contains no other variables.
        Meanwhile,  $\Bound(\tbullet,  
        p(\term_1, \ldots, \term_n)) 
        = \emptyset$.   
    \item  
                $\Free(\tbullet, 
                [\bigwedge_{i=1}^{l (\ge 2)}
                P_{i}(t^{i}_1,\ldots,t^{i}_n)](t_1, \ldots, t_n))$ 
                contains any variable 
                $\delta(v \oplus \tbullet)$ 
                if $v$ occurs in 
                $t_j, 1 \le j \le n$. 
                It  also contains 
                $\bigcup_{i=1}^l 
                \Free(t_1. \cdots .t_n 
                \oplus \tbullet, 
                P_{i}(t^{i}_1, \ldots, 
                t^{i}_n))$. 
                It contains no other variables.  
                Meanwhile, 
                {\small $\Bound(\tbullet,  
                [\bigwedge_{i=1}^{l (\ge 2)} P_{i}(t^{i}_1,\ldots,t^{i}_n)](t_1, \ldots, t_n)) = 
\bigcup_{i=1}^l \Bound(t_1.\cdots .t_n \oplus \tbullet, P_{i}(t^{i}_1, \ldots, 
t^{i}_n))$.} 
    \item  
                $\Free(\tbullet,  
                [P \gtrdot F](t_1, \ldots, t_n)) 
                = \Free(\tbullet,  
                P(t_1, \ldots, t_n)) 
                \cup \Free(t_1. \cdots .t_n \ \oplus \ 
                \tbullet, F)$. 
        Meanwhile, 
        {\small $\Bound(\tbullet,  
        [P \gtrdot F](t_1, \ldots, t_n)) 
        = \Bound(\tbullet, 
         P(t_1, \ldots, t_n)) 
         \cup \Bound(t_1. \cdots .t_n\ \oplus \ 
         \tbullet, F)$.} 
    \item $\Free(\tbullet,  
        \neg F) = \Free(\tbullet, F)$. And 
        $\Bound(\tbullet, \neg F) 
        = \Bound(\tbullet, F)$.    
    \item $\Free(\tbullet,  \forall v.F)  
        = (\Free(\tbullet, F) \cup 
        \{\delta(v_x \oplus \tbullet) \ | \ v_x \in \Subvariable(v)\ {\andC}\\ \leftmost(v_x) \not= 
            \leftmost(v) \})$. 
        Meanwhile,  
        $\Bound(\tbullet, \forall v.F) 
     = \Bound(\tbullet, F) \cup  
     \{{\delta(v_y \oplus \tbullet)}\ | \ 
         v_x \in \Subvariable(v) \ \andC\ 
         v_y = v_x[\leftmost(v_x) \mapsto 
         \leftmost(v)]\}$. Here, 
     $v_x[\leftmost(v_x) \mapsto \leftmost(v)]$ 
     is a variable that is almost exactly 
     $v_x$ except that $\leftmost(v_x)$ 
     is replaced by $\leftmost(v)$. 
     \item $\Free(\tbullet,  \exists v.F)  
         = \Free(\tbullet, \forall v.F)$. 
         And 
         $\Bound(\tbullet, \exists v.F) 
         = \Bound(\tbullet, \forall v.F)$. 
    \item  
        {\small $\Free(\tbullet,  
        F_1 \wedge F_2) = 
        \Free(\tbullet,  F_1 \vee F_2) 
        = \Free(\tbullet, F_1 \supset F_2)
        = \Free(\tbullet,  F_1) 
        \cup \Free(\tbullet,  F_2)$.} \\ 
{\small $\Bound(\tbullet,  
        F_1 \wedge F_2) = 
        \Bound(\tbullet,  F_1 \vee F_2) 
        = \Bound(\tbullet,  F_1 \supset F_2)
        = \Bound(\tbullet, F_1) 
        \cup \Bound(\tbullet,  F_2)$.} $\Box$ 
        \end{enumerate}    
\end{definition}  
\begin{definition}[Well-formedness]   
    \normalfont
    Let $\Subformula$ be a function 
       that takes a formula 
       and that outputs the set of 
       all the subformulas of the formula.  
       I assume that  
       any $p(\term_1, \ldots, \term_n)$, 
       $\top$, $\bot$, any $t_1 = t_2$ 
       or any $t_1 \simeq t_2$ has 
       only itself as its subformula, 
       but otherwise 
       I take it for granted that 
       readers are familiar with the notion 
       of a subformula of a formula; Cf. 
       \cite{Kleene52}. 
         Let $\pnode$ be a function that takes two formulas and that outputs a set of formulas such that
         $\pnode(F_1, F_2) =  
         \{[P \gtrdot F_3](\term_1, \ldots, 
             \term_n) \in 
             \Subformula(F_1) \ | \  
             F_2 \in \Subformula(F_3)\ 
             {\andC}\ 
             (\text{there exist no } P_x, F_4 
             \text{ such that } 
              [P_x \gtrdot F_4](\term_1, 
              \ldots, \term_m) \in 
              \Subformula(F_3)\ {\andC}\ 
              F_2 \in \Subformula(F_4))
         \}$.  
I say that a formula $F$ is well-formed  
    iff the following conditions hold true 
    of $F$. 
    \begin{description} 
        \item[(Term compatibility in equalities)]    
            $\compatible(F)$ which holds 
            true iff 
            the following satisfy for any subformula 
            of $F$ in either of the  
            forms: $t_1 = t_2$ or $t_1 \simeq t_2$: 
   \begin{enumerate}
       \item If $t_1$ is 
           either 
           some basic variable or 
           some basic domain 
           symbol, 
           then 
           $t_2$ is either 
           a basic variable or a 
           basic domain symbol. 
       \item If $t_1$ is  $t'_1 \lessdot \tstarLength{m}_1$ 
           for some term $t'_1$ and 
           some term body $\tstarLength{m}_1$, 
                   then 
           (A)  
           $t_2$ is $t'_2 \lessdot \tstarLength{m}_2$ 
           for some term $t'_2$ and 
           some term body $\tstarLength{m}_2$; 
           (B) $\compatible(t'_1 = t'_2)$; 
            and 
           (C) 
           $\compatible(\tstarLength{m,i}_{1} =
           \tstarLength{m,i}_{2})$ 
           for each $1 \le i \le m$.  
   \end{enumerate}    
     \item[(Connector compatibility)]     
            $\Ccompatible(F)$ which holds 
            true 
      iff, for any basic predicate $p(\term_1, \ldots, \term_n)$  
      occurring as a subformula of $F$, 
      and for any $\term_i$, $1 \le i \le n$,  
      if $\term_i = {^\circ t_i}$ 
      for some term $t_i$, 
      then 
      for each $[P \gtrdot F_3](t'_1, \ldots, 
      t'_m) \in 
      \pnode(F, p(\term_1, \ldots, \term_n))$, 
      there exists $1 \le j \le m$ 
      such that $t'_j = t_i$. 
\item[(No free variables)]  
    $\Free(\epsilon^{\dagger}, F) = \emptyset$. 
         $\Box$ 
   \end{description}    
   \end{definition}  
   Connector compatibility says 
   that the use of 
   $^\circ t_i$ must be proper (Cf. 1.3): 
   it must have a corresponding term $t_i$.  
\subsection{Semantics}    
Let us define domains of discourse. 
Since I presume that any object 
is divisible by attributes 
it has, the domains must 
be nested ones. I have mentioned 
also that an object may be a collection 
of objects, so it can be also divisible 
by objects. The divisibility 
and the nesting then go in two directions. \\
\begin{adjustwidth}{1cm}{}
But, let us begin with some intuition 
on divisibility by attributes, and then 
on divisibility by objects. 
For a start readers could 
imagine a countable set $W$ of possible worlds. 
To make matters even simpler, let us for now forget about divisibility 
by objects. Readers could (for now) imagine that 
an object in predicate gradual logic is 
each such possible world. 
Usually it is the case that a possible 
world $w$ is a point in $W$ which is not further 
divisible. In predicate gradual logic  
$w \in W$ is a point {\it only in the particular
    set $W$}. That is, if we look into $w$, 
what appeared as a point in $W$ turns out to 
have extension, like in any fractal space. 
This should form primitive understanding 
of $\lessdot$ (and also $\gtrdot$), 
which is a relation between 
objects that are in one space and their attributes
that are in other attributive spaces 
of the space. This, however, 
is of course too coarse to 
capture domains of discourse in predicate gradual logic.
Rigorously, any combination of 
points in such a Kripkean space, 
or in fact any combination of sequences of points 
in any Kripkean spaces should be permitted to 
have such attributive 
Kripkean spaces in order that that be 
reflection of predicate gradual logic' objects 
and attributes. We have covered up to 
divisibility by attributes. Let us now 
add
divisibility by objects into this understanding. 
We no longer just have those points, 
but we can also refer to collections 
of those points as if they were 
points. Remember the entities in set theory, where
every one of them is a set. We are done 
with divisibility by objects. 
Now, to put things together, 
any collections of points, any sequence 
of them in one or more such spaces 
can have an attributive space, which 
is a proper understanding about 
the divisibility and the nesting of objects in predicate 
gradual logic. \\
\end{adjustwidth}  
\indent Incidentally, by an object I do not mean 
the usual indivisible entities lying in the external 
world. The postulate of such entities and the postulate 
that any reasoning would be possible 
of them are incompatible with my point of view \cite{Arisaka15Gradual}.\\
\indent Onto formal definitions now. 
I define a semantic structure 
to be a tuple ${(\mathcal{D}, \mathfrak{D}, 
    \mathfrak{P}, \mu, \mathcal{I}_{\mu}, \sqsubseteq, 
    \in^\star)}$. 
The components are as defined below.    
\begin{definition}[$\mathcal{D}$: Constituents 
    of domains of discourse]   
    \normalfont
    Let $\mathcal{D}^-$ be 
    an at most countable set. 
    Its elements 
    are referred to by $d$ with 
    or without a subscript.  
    Then $\mathcal{D}$ is 
    some subset of  
    $\mathcal{D}^- \times \mathbb{N}$. 
    I refer to any $(d, n) \in \mathcal{D}$ by $d^n$. 
    I define $\rho$ to be a projection 
    operator for $\mathcal{D}$ into $\mathcal{D}^-$ 
    such that: $\rho(d^n) = d$ if $d^n \in \mathcal{D}$.
    $\Box$ 
\end{definition}
\hide{ 
But before that, one thing I should note is that  
I will be liberal about the use of $=$. 
In some cases the symbol appeals to mean 
the equality in ordinary arithmetical sense: 
(some A) = (some B) means that A is B and B is A. 
For some ordered pairs it is formally defined 
to mean a certain thing. 
The convention is: 
if $=$ is used for ordered pairs for which 
a formal definition of $=$ is given, 
then the definition applies; otherwise, 
the ordinary arithmetical sense applies. 
} 
From here on I will call each element of $\mathcal{D}$ 
strictly by a {\it basic object}. 
It is not 
necessary that a basic object 
be independent of the others. It could happen 
that $d^{n_1}_1 \in \mathcal{D}$ be a collection 
of $d^{n_2}_2, d^{n_3}_3 \in \mathcal{D}$. It, however, 
must be 
a set in ZFC; and in case 
$d^{n}$ comprises $d^{n_1}_1, \ldots, d^{n_k}_k$, 
it must be that $d_1 \not= \cdots \not= d_k$. 
There is a good reason why it should be so, 
which will be explained later. \\
\indent It will become convenient that 
we 
have similar notions   
to term symbols, term body symbols, and so on  for the elements of $\mathcal{D}$, too. 
Let us define object symbols, $o$ with or without 
a subscript and/or a superscript, and object body 
symbols, $\ostar$ with or without a subscript 
and/or a superscript. They derive 
from substituting $o$ into all $t$ and 
$\ostar$ into all $\tstar$ in the definitions 
of \textbf{Term body} and \textbf{Term}.  
\hide{ 
\begin{description} 
    \item[Object body symbol]{\ }
        \begin{enumerate} 
            \item A sequence of object symbols 
                $o_1.\cdots.o_n$ for $n \ge 1$.
        \end{enumerate}
    \item[Object symbol]{\ } 
        \begin{enumerate} 
            \item A basic object symbol. 
            \item $o \lessdot \ostar$  
                for some object symbol $o$ 
                and some object body symbol $\ostar$.
        \end{enumerate} 
\end{description}   
}
I denote the class of all object symbols 
by $\mathfrak{O}$. 
$\ostarLength{m}$ with or without a subscript 
similarly denotes a sequence of 
object symbols with length $m$, and, similarly, marked object bodies (Cf. 
around 
Definition 4), $\obullet$ with or without a subscript. 
Similarly for $\oplus$, $\tau_l$, $\tau_r$ and $\delta$, 
and assume also $\Subobject$ is analogous to 
$\Subterm$. 
\hide{ 
\begin{definition}[$\mu$] 
    \normalfont The function $\mu$ 
    is recursively defined. \ryuta{Everything 
        should be an object, check later again. 
    But I do not have to define 
it for a term body symbol, actually.} 
    \begin{enumerate}
        \item $\mu(x) \in \mathcal{D}$ for 
            any $x \in \mathcal{X}$.
        \item $\mu(a) = a$.  
        \item $\mu(p^0) = p^0$.  
        \item $\mu(\neg p^0) = \neg p^0$. 
        \item $\mu(\top) = \top$. 
        \item $\mu(\bot) = \bot$. 
        \item $\mu(t \lessdot \tstar) 
            = \mu(t) \lessdot \mu(\tstar)$. 
        \item $\mu(t_1.\cdots.t_n) 
            = \mu(t_1).\cdots.\mu(t_n)$. 
    \end{enumerate} 
    $\Box$ 
\end{definition}   
}
\begin{definition}[$\mu$: Variable assignment  and $\mathcal{I}_{\mu}$: Interpretation]  
    \normalfont   
    $\mu$ is defined to be some function 
    that takes as inputs 
    an object body and a basic 
    variable symbol and that outputs 
    a basic object. 
    $\mathcal{I}_{\mu}$ is a 
            function dependent on 
            some variable assignment $\mu$, defined as follows.
\begin{enumerate}
             \item $\mathcal{I}_{\mu}(\delta(\obullet), 
                     a) \in \mathcal{D}$ 
                     for each basic domain symbol $a$.
                 \item $\mathcal{I}_{\mu}(\delta(\obullet), 
                     x) = \mu(\delta(\obullet), x)$. 
                 \item $\mathcal{I}_{\mu}(\delta(\obullet), 
                     t_1.\cdots.t_n) 
                     = \mathcal{I}_{\mu}(\delta(\obullet), 
                     t_1).\cdots.\mathcal{I}_{\mu}(\delta(\obullet),
                     t_n)$. 
                 \item $\mathcal{I}_{\mu}(\delta(\obullet), 
                     t \lessdot \tstar) 
                     = \mathcal{I}_{\mu}(\delta(
                     \Imu(\delta(\obullet),\tstar)  \oplus \obullet), t) \lessdot \mathcal{I}_{\mu}(\delta(\obullet), \tstar)$.    
                 \item $\mathcal{I}_{\mu}(\delta(\obullet), 
                     {^\circ t}) =  \mathcal{I}_{\mu}(
                     \delta(\tau_r(\obullet)), t)$.  
                 \item $\mathcal{I}_{\mu}(\delta(\obullet), 
                     p) \subseteq  
                     \mathfrak{O}^n$ 
                     for any $n$-ary basic predicate. 
                     $\Box$ 
             \end{enumerate} 
\end{definition} 

\begin{definition}[$\sqsubseteq$: Binary relation 
    over object body symbols]  
    \normalfont The binary relation satisfies the following 
    conditions.  
\begin{adjustwidth}{1cm}{}

\begin{description}   
                     \item[Reflexibity] 
                         For any object body 
                         symbol $\ostar$, 
                         it holds that 
                         $\ostar \sqsubseteq 
                         \ostar$.    
                     \item[Transitivity] For any three 
                         object body symbols 
                         $\ostar_1, \ostar_2$ 
                         and $\ostar_3$, 
                         it holds that 
                         $\ostar_1 \sqsubseteq \ostar_3$
                         if $\ostar_1 \sqsubseteq 
                         \ostar_2$ 
                         {\andC} $\ostar_2 \sqsubseteq 
                         \ostar_3$. 
                     \item[Antisymmetry] 
                         For any two 
                         object body symbols 
                         $\ostar_1$ and $\ostar_2$, 
                         it holds that 
                         $\ostar_1 = \ostar_2$  
                         iff  
                         $\ostar_1 \sqsubseteq \ostar_2$
                         {\andC} $\ostar_2 \sqsubseteq 
                         \ostar_1$. 
                     \item[Subsumptivity] For  
                         any three 
                         object body symbols 
                         $\ostar_1, \ostar_2$ and 
                         $\ostar_3$, it holds  
                         that 
                         $(\ostar_1 \lessdot 
                         \ostar_2) \sqsubseteq 
                         \ostar_3$ 
                         iff $\ostar_2 \sqsubseteq 
                         \ostar_3$. 
                         It also holds that 
                         $\ostar_1 \sqsubseteq 
                         (\ostar_2 \lessdot 
                         \ostar_3)$ iff 
                         $\ostar_1 = 
                         (\ostar_4 \lessdot \ostar_5)$
                         for some $\ostar_4$ and 
                         $\ostar_5$ 
                         such that 
                         $\ostar_5 
                         \sqsubseteq \ostar_3$ 
                         {\andC} $\ostar_4 \sqsubseteq 
                         \ostar_2$. 
                         \hide{ 
                       \item[Parallel]  
                        For any 
                        four domain body 
                        symbols $\ostar_1, \ostar_2, 
                        \ostar_3$
                        and $\ostar_4$,  
                        it holds that 
                        $\ostar_1 \prec \ostar_2 
                        \sqsubseteq 
                        \ostar_3 \prec \ostar_4$ 
                        iff $\ostar_2 = \ostar_4$ 
                        {\andC} $\ostar_1 \sqsubseteq 
                        \ostar_3$. 
                    }
                     \item[Sequentiality]   
                     For any two 
                         object body symbols 
                         $\ostarLength{n}_1$ and 
                         $\ostarLength{m}_2$, 
                         it holds
                         that $\ostarLength{n}_1
                         \sqsubseteq \ostarLength{m}_2$ 
                     iff 
                     $n = m$ {\andC} 
                         $\ostarLength{n,i}_1
                             \sqsubseteq 
                             \ostarLength{m,i}_2$ 
                             for each $1 \le i \le m$. 
                             $\Box$ 
                 \end{description} 
             \end{adjustwidth}  
        \end{definition}    
        \begin{definition}[ 
            $\mathfrak{D}$: Domain function, 
            and $\mathfrak{P}$: Predicate function]   
            \normalfont
            Components of a structure around 
            domains are defined to satisfy the following 
            conditions. $\mathfrak{D}$ takes 
            as the input $\epsilon$ or an 
            object body and outputs a subset of 
            $\mathcal{D}$. 
            $\mathfrak{P}$ takes as the input 
            $\epsilon$ or an object body 
            and outputs a subset of $\mathcal{P}$. 
\begin{adjustwidth}{1cm}{}
        \begin{description}   
            \item[Top]  
                $\mathfrak{D}(\epsilon) = 
                \mathcal{D}$ 
                {\andC}  $\mathfrak{P}(\epsilon) = 
               \mathcal{P}$. 
    \item[Monotonicity]   
        For any two object body symbols 
        $\ostar_1, \ostar_2$, 
        it holds that 
        if 
         $\ostar_1 \sqsubseteq \ostar_2$, 
        then $\mathfrak{D}(\ostar_1) \subseteq \mathfrak{D}(\ostar_2)$ 
        {\andC} $\mathfrak{P}(\ostar_1) \subseteq 
        \mathfrak{P}(\ostar_2)$.  $\Box$ 
\end{description}   
\end{adjustwidth} 
        \end{definition}    
This definition embodies attribute normal interpretation
for both $\mathcal{D}$ and $\mathcal{P}$. 
Now, let us come back to the point made 
after Definition 8. The point of the superscripts on some $d$ is 
to express differences between $d^1$, $d^2$, ..., 
that are observable if we zoom out of 
the domain of discourse that contains them, say 
$\mathfrak{D}(\ostar)$, 
but that are not observable in the very domain 
of discourse. Hence in the very domain $\mathfrak{D}(\ostar)$ having 
a set $\{d^1, d^2\}$ for instance is the same as 
having another set $\{d^i\}$ for some 
$d^i \in \mathfrak{D}(\ostar)$. 
\hide{ 
\noindent The condition \textbf{(No empty entities
    in the outermost domain of discourse)} basically says that syntax
and semantics are not entirely 
separated within a coherent structure to 
a given formula. It reflects 
the viewpoint of mine 
stipulated in the previous section. 
\begin{proposition} 
   If $\alstar_1$ does not have 
   the same structure as $\alstar_2$, 
   then it does not hold that 
   $\alstar_1 = \alstar_2$. 
\end{proposition}  
\begin{proof} 
  wnoenownen
\end{proof}  
} 
\begin{definition}[$\in^{\star}$: Indexed set inclusions]  
    \normalfont
    I define a class of 
    set inclusion operators: $\{\in\} \times  
    \mathfrak{O}$ such that  
    if $\ostar_1 (\in, \delta(\obullet_1)) \ostar_2$, 
    it is not necessary that
    $\ostar_1 (\in, \delta(\obullet_2)) \ostar_2$ 
    unless $\delta(\obullet_1) = 
    \delta(\obullet_2)$. Apart from this mutual 
    independence, each  
    one of them 
    is a ZFC set inclusion operator.  
    I denote any $(\in, \delta(\obullet))$ 
by $\in^{\delta(\obullet)}$. 
     I define $\in^{\star}$ to be 
     the class of all these set inclusions. 
    $\Box$ 
\end{definition}   
Incidentally, suppose there are two sequences of 
object symbols. One is $\ostar_1 = d_1^{n_1}.\cdots.d_i^{n_i}$, 
and the other is $\ostar_2 = d_1^{n'_1}.\cdots.d_i^{n'_i}$ for $i > 1$.  
Even if $d_j^{n_j} \in^{\delta(\obullet)} d_j^{n'_j}$ 
for at least one $1 \le j \le i$ (but potentially 
all such $j$) and 
$d_k^{n_k} = d_k^{n'_k}$ 
for any other $1 \le k \le i$, 
it is not perforce the case that  
$\ostar_1 \in^{\delta(\obullet)} \ostar_2$, while, 
of course, if $\ostar_1 \in^{\delta(\obullet)} 
\ostar_2$ $\andC$ $\ostar_2 \in^{\delta(\obullet)} 
\ostar_3$, then it holds true that 
$\ostar_1 \in^{\delta(\obullet)} \ostar_3$.   

I have just finished defining all the components 
of a structure. There are a few other definitions required for 
characterisation of semantics.   
\begin{definition}[Variable assignment updates]  
    \normalfont  
    Let $\xi^n$ be a recursive function  
    that takes a variable and that outputs  
    a marked term body, such that:  
    $\xi^0(v)$ is almost exactly $v$ except 
    that $\leftmost(v)$ is replaced by 
    $(\leftmost(v))^{\dagger}$; 
    and $\xi^{k+1}(v) = \tau_r(\xi^k(v))$. 
    Let $\mathcal{V}(v)$ be 
    $\bigcup_{k \ge 1} \{\xi^k(v) \ | \ \xi^k(v) 
        \not= \epsilon^{\dagger}\}$. 
  Then I define $\update$ to be a function 
  taking: a structure; an object body symbol or 
  else $\epsilon$; a  variable; and a basic object, 
  and 
  outputting a structure, such that 
  $\update(\mathcal{M}, \delta(\obullet), 
  v, d^n)$ is a strucutre that is almost 
  exactly $\mathcal{M}$ except that   
  $\mu(\delta(\obullet), \leftmost(v)) = d^n$ 
  and  
  $\mu(\delta(\Imu(\delta(\obullet), 
  \delta(\tbullet))  
  \oplus \obullet), \leftmost(v)) = d^n$  
  for each $\tbullet \in \mathcal{V}(v)$.  $\Box$ 
\end{definition}  
This embodies the attribute normal interpretation 
for variable assignment. 
\begin{definition}[Surface equality]      
    \normalfont 
    I say that $o_1 \simeq o_2$  
    for any two objects $o_1$ and $o_2$  
    iff $|\Subobject(o_1)| = |\Subobject(o_2)|$ 
    $\andC$ 
    $\rho(\leftmost(o^1)) = \rho(\leftmost(o^2))$ where 
    $o^1$ ($o^2$) is the last element in 
    $\Subobject(o_1)$ (in 
    $\Subobject(o_2)$) 
    $\andC$ if $|\Subobject(o_1)| \not= 1$, then 
    for all $1 \le i < |\Subobject(o_1)|$, 
    \{$i$-th element of $\Subobject(o_1)$\} 
    = \{$i$-th element of $\Subobject(o_2)\}$.  
    $\Box$ 
\end{definition}   
\begin{definition}[Proper objects] 
    \normalfont
     I define $\properO$ to be a predicate dependent 
     on some $\mathfrak{D}$. It takes 
     an object body symbol or $\epsilon$, 
     and an ordered list of object body symbols. 
     The definition is recursive. 
     \begin{enumerate} 
         \item $\properO(\delta(\obullet), (d^n))$ 
             iff $d^n \in \mathfrak{D}(\delta(\obullet))
             $. 
         \item $\properO(\delta(\obullet), (o 
             \lessdot \ostar))$ iff 
             $\properO(\delta(\ostar \oplus \obullet), (o))$ $\andC$ $\properO(\delta(\obullet), (\ostar))$.  
         \item $\properO(\delta(\obullet), 
             (o_1.\cdots.o_n))$ iff 
             $\properO(\delta(\obullet), (o_i))$ 
             for each $1 \le i \le n$.  
         \item $\properO(\delta(\obullet), 
             (\ostar_1, \ldots, \ostar_n))$ iff 
             $\properO(\delta(\obullet), 
             (\ostar_i))$ 
             for each $1 \le i \le n$. $\Box$ 
     \end{enumerate} 
\end{definition} 

\begin{definition}[Proper g-terms]  
    \normalfont
    I define 
    $\proper$ to be a predicate   
dependent on some $\mathcal{I}_{\mu}$. 
    It takes an object body symbol or $\epsilon$, 
    and a g-term. It is recursive.  
    \begin{enumerate} 
        \item $\proper(\delta(\obullet), x)$ 
            iff $\Imu(\delta(\obullet), x)  
            \in \mathfrak{D}(\delta(\obullet))$. 
        \item $\proper(\delta(\obullet), a)$ 
            iff $\Imu(\delta(\obullet), a) 
            \in \mathfrak{D}(\delta(\obullet))$. 
        \item $\proper(\delta(\obullet), t \lessdot \tstar)$ 
            iff {\small $\proper(\delta(\obullet), \tstar)$ 
                $\andC$ $\proper(\delta(\Imu(\delta(\obullet), \tstar) \oplus \obullet), t)$.}  
        \item $\proper(\delta(\obullet), t_1.\cdots.t_n)
            $ iff $\proper(\delta(\obullet), t_i)$ 
            for all $1 \le i \le n$.  
        \item $\proper(\delta(\obullet), ^\circ t)$ 
            iff $\proper(\delta(\tau_r(\obullet)), t)$. 
            $\Box$ 
    \end{enumerate} 
\end{definition}  
These two notations judge whether each basic predicate 
and term can appear under a given 
domain of discourse. 
\begin{definition}[Semantics]   
    \normalfont  
    I define the following forcing relations. 
\begin{enumerate}     
        
    \item[A] $\mathcal{M}, \obullet 
        \models p(\term_1, \ldots, \term_n)$ 
       is $\mathbb{U}$ 
       if $\mathcal{I}_\mu(\delta(\obullet), p) = \emptyset$ $\orC$ 
       $\notC$ $\properO(\delta(\obullet),
       (\Imu(\delta(\obullet), p)))$
           $\orC$ 
           $p \not\in \mathfrak{P}(\delta(\obullet))$ 
           $\orC$ $\notC$ 
           $\proper(\delta(\obullet), \term_i)$ 
           for some $1 \le i \le n$. 
           These rules marked with alphabets 
       have a priority 
       over any other rules to follow 
       and override any of them.  
   \item[B] $\mathcal{M}, \obullet 
       \models t_1 = t_2$ is $\mathbb{U}$ 
       if $\notC\ \proper(\delta(\obullet),t_1)
       \ \orC \ \notC\ \proper(\delta(\obullet),t_2)$. 
   \item[C] $\mathcal{M}, \obullet 
       \models t_1 \simeq t_2$ is $\mathbb{U}$ 
       if $\notC\ \proper(\delta(\obullet),t_1) \ \orC \ \notC\ \proper(\delta(\obullet),t_2)$. 
   \item[D] $\mathcal{M}, \obullet 
       \models  [\bigwedge_{i=1}^{l (\ge 2)}P_i(t_1^i, 
       \ldots, t^i_n)](t_1, \ldots, t_n)$ is $\mathbb{U}$ 
       if there exist some $1 \le i \le l$ 
       and some $1 \le j \le n$ such that 
       $\notC\ \proper(\delta(\obullet), t_j^i)$ 
       $\orC$ 
        $\notC\ \proper(\delta(\obullet), 
       t_j)$. 
   \item[E] $\mathcal{M}, \obullet 
       \models [P \gtrdot F](t_1, \ldots, t_n)$
       is $\mathbb{U}$ if 
       there exists some $1 \le j \le n$ such that 
       $\notC\ \proper(\delta(\obullet), t_j)$ 
       $\orC$ $\notC$ $\mathcal{M}, 
       \Imu(\delta(\obullet), t_1).\cdots.
       \Imu(\delta(\obullet), t_n) 
       \oplus \obullet \models F$.
       \\
\item $\mathcal{M}, 
        \obullet \models p(\term_1, \ldots, \term_n)$ 
        if $(\mathcal{I}_{\mu}(\delta(\obullet), \term_1), 
        \ldots, \mathcal{I}_{\mu}(\delta(\obullet), \term_n)) 
        \in^{\delta(\obullet)} \mathcal{I}_{\mu}(\delta(\obullet), p)$. 
    \item $\mathcal{M}, \obullet 
        \models [\bigwedge_{i=1}^{l (\ge 2)}
        P_{i}(t^{i}_1,
        \ldots, t^{i}_n)](t_1, \ldots, t_n)$ 
        if:\\ 
            $\mathcal{M}, 
            \obullet 
            \models P_i(t^i_1,\ldots, 
            t^i_n)  
            $  
            for each $1 \le i \le l$, 
            $\andC$ 
            $
            \Imu(\delta(\obullet), t_1^i).\cdots.\Imu(\delta(\obullet), t_n^i) 
            \in^{\delta(\obullet)} \Imu(\delta(\obullet), t_1).\cdots. \Imu(\delta(\obullet), t_n)$ 
            for each $1 \le i \le l$.  
    \item $\mathcal{M}, \obullet 
        \models [P \gtrdot F](t_1, \ldots, t_n)$ 
        if: \\
            $\mathcal{M}, \obullet 
            \models P(t_1, \ldots, t_n)$ 
                {\andC} 
                $\mathcal{M},   \Imu(\delta(\obullet), t_1).\cdots.\Imu(\delta(\obullet), t_n) \oplus \obullet \models F$.  

    \item $\mathcal{M}, \obullet \models \top$. 
    \item  $\notC$ $\mathcal{M}, \obullet \models \bot$. 
    \item $\mathcal{M}, \obullet \models 
        \neg F$ if 
        $\notC\ \mathcal{M},\obullet \models 
        F$. 
    \item $\mathcal{M}, \obullet \models 
        \forall v.F$ if 
        {\small $\update(\mathcal{M}, 
        \delta(\obullet), v, d^n), \obullet
         \models  
         F$} for each 
     {\small $d^n \in \mathfrak{D}(\delta(\Imu(\delta(\obullet), v) \oplus \obullet))$}.  
    \item $\mathcal{M}, \obullet \models 
        \exists v.F$ if 
        {\small $\update(\mathcal{M}, 
        \delta(\obullet), v, d^n), \obullet
         \models  
         F$} 
        for some   {\small $d^n \in \mathfrak{D}(\delta(\Imu(\delta(\obullet), v) \oplus \obullet))$}.  
    \item $\mathcal{M}, \obullet \models 
        F_1 \wedge F_2$ if 
        $\mathcal{M}, \obullet \models F_1$ 
        {\andC} $\mathcal{M},\obullet \models F_2$. 
    \item $\mathcal{M}, \obullet \models F_1 \vee F_2$ 
        if $\mathcal{M},\obullet \models F_1$ {\orC} 
        $\mathcal{M},\obullet \models F_2$.  
    \item $\mathcal{M}, \obullet \models F_1 \supset F_2$
        if $\mathcal{M}, \obullet \models 
        \neg F_1$ 
        {\orC} $\mathcal{M}, \obullet \models F_2$.  
    \item $\mathcal{M}, \obullet \models t_1 = t_2$ 
        if $\Imu(\delta(\obullet), 
        t_1) = \Imu(\delta(\obullet),
        t_2)$. 
    \item $\mathcal{M}, \obullet \models t_1 \simeq t_2$ 
        if $\Imu(\delta(\obullet), t_1) 
     \simeq   
    \Imu(\delta(\obullet), t_2)$. $\Box$ 
\end{enumerate}   
\hide{ 
$\normal$ is defined to be a function 
that takes a formula 
and that outputs a formula: 
\begin{enumerate}[label=(\Alph*)] 
    \item $\normal(F) = F$ iff 
        $F \not= \neg F_1$ 
        for some $F_1$. 
    \item $\normal(\neg p) = p^-$
    \item $\normal(\neg [\bigwedge_{i=1}^{l (\ge 2)} 
        P_i(t_1^i, \ldots, t_n^i)](t_1, 
        \ldots, t_n)) = \\
        (\bigvee_{i = 1}^l\neg 
        P_i(t_1^i \prec t_1.\cdots.t_n, 
    \ldots, t_n^i \prec t_1.\cdots.t_n))$.  
\item $\normal(\neg [P \gtrdot F](t_1, \ldots, t_n))
    = $. 
\end{enumerate}
}   
\end{definition}   
Alphabetical rules apply when a predicate 
or a term that is not under $\delta(\obullet)$ 
nonetheless appears. To explain 
the second condition of the rule E, 
$P \gtrdot F$ is inconsistent in that case, 
and the terms are then predicated 
by inconsistency, which is impossible. 
\begin{definition}[Coherence]  
    \normalfont
    Let $\mathcal{M}$ be a structure. Meanwhile, let 
    $F$ be a well-formed formula in $\mathcal{L}$.    
   I say that $\mathcal{M}$ coheres to $F$ 
   iff 
   $\mathcal{M}, \epsilon^{\dagger} \models F$ 
   is either $\mathbb{T}$ or 
   $\mathbb{F}$.  $\Box$
\end{definition} 
\begin{definition}[Satisfiability and 
    validity]  
    \normalfont 
   I say that a structure $\mathcal{M}$
   satisfies a formula $F$ 
   iff $\mathcal{M}$ coheres to $F$ $\andC$
   $\mathcal{M}, \epsilon^{\dagger} \models F$. 
   I say that $F$ is valid iff 
   there is at least one $\mathcal{M}$ 
   that satisfies $F$ {\andC}  
   every $\mathcal{M}$ that coheres 
   to $F$ satisfies $F$. 
   $\Box$ 
\end{definition}   
\begin{theorem}[Consistency in finite cases] 
    Restricted to finite $\mathcal{D}$, 
   predicate gradual logic is consitent.   
   \end{theorem} 
   \begin{proof}  
       Mostly straightforward. 
   \end{proof} 
\section{Revisiting the Mac sentence}  
\begin{enumerate}[label={(\protect\perhapsstar\arabic*)}] 
        \setcounter{enumi}{5}
    \item Every school which employs a 
        handsome PC instructor  
         who teaches every 
         female learner using a Mac machine 
         buys it for him. 
         \setcounter{enumi}{5} 
    \staritem  {\small $\forall x\ \exists y \lessdot x \ 
            \forall z \lessdot y \lessdot x\ \exists w \lessdot z \lessdot y \lessdot x\ \exists u.([\school \gtrdot 
            ([\pcinstructor \gtrdot \handsome(^\circ y) \wedge ([\student 
            \gtrdot \female(^\circ z) \wedge 
            (\Mac(w) \wedge  
            \use(^\circ z, w))
        ](z) 
            \supset \teach(^\circ y,z))
            ] (y) \wedge \employ(^\circ x,y) )]
            (x) \supset w \simeq u \wedge  
            \buy(x, u))$}. 
\end{enumerate}  
We may have the following semantic structure 
$(\mathcal{D}, \mathfrak{D}, \mathfrak{P}, \mu, 
\Imu, \sqsubseteq, \in^{\star})$: 
\begin{description} 
    \item[$\mathcal{D}$] is 
        $\{$school, PCinstructor, learner,
            Mac$\}$. 
    \item[$\mathfrak{D}$] satisfies: 
        \begin{enumerate} 
            \item PCinstructor $\in \mathfrak{D}($school$)$.
            \item learner $\in 
                \mathfrak{D}($PCinstructor$\lessdot$ school$)$. 
            \item Mac $\in 
                \mathfrak{D}($learner $\lessdot$ PCinstructor$\lessdot$school$)$. 
        \end{enumerate}
    \item[$\mathfrak{P}$] satisfies: 
        \begin{enumerate} 
            \item $\school, \employ \in \mathfrak{P}(\epsilon)$.
            \item $\pcinstructor, \teach \in \mathfrak{P}($school$)$. 
            \item $\handsome, \student 
                \in \mathfrak{P}($PCinstructor$\lessdot$school$)$. 
            \item $\female, \Mac, \use 
                \in \mathfrak{P}($learner $\lessdot$ PCinstructor$\lessdot$school$)$. 
        \end{enumerate}
    \item[$\Imu$] satisfies: 
        \begin{enumerate} 
            \item $\Imu(\epsilon, \school) = \{$school$\}$. 
            \item $\Imu($school$, \pcinstructor) = 
                \{$PCinstructor$\}$. 
            \item $\Imu($school$, \teach) = 
                \{($PCinstructor, learner$)\}$. 
            \item $\Imu($PCinstructor $\lessdot$school,
                \handsome
                $) = \{$PCinstructor$\}$. 
            \item $\Imu($PCinstructor $\lessdot$school,
                \student$) = \{$learner$\}$. 
            \item $\Imu($learner $\lessdot$ PCinstructor
                $\lessdot$ school, 
                $\female) = \{$learner$\}$.
            \item $\Imu($learner $\lessdot$ PCinstructor
                $\lessdot$ school,
                $\Mac) = \{$Mac$\}$. 
            \item $\Imu($learner $\lessdot$ PCinstructor
                $\lessdot$ school, 
                $\use) = \{$(learner, Mac)$\}$. 
        \end{enumerate} 
\end{description} 
Then it satisfies the well-formed formula 
($\star$6). 
\section{Predicate Gradual Logic and Aristotle's Syllogisms}   
In the other work on propositional logic
\cite{Arisaka15Gradual},  
I mentioned that quantification 
will be required to represent Aristotle's syllogisms. 
I show that basically all that play 
a major part in Aristotle's discourse 
in Prior Analytics are represented in 
predicate gradual logic.  In this section, 
those in 
the form: $\forall x.P_b(x) \supset P_a(x)$, 
are universal affirmative expressions, 
meant to be read: Every $b$ is $a$ (A in 
Figure 1), those in 
the form: $\forall x.P_b(x) \supset \neg P_a(x)$ 
are universal negative expressions, meant 
to be read: No $b$ is $a$ (E in Figure 1), 
those in the form: $\exists x.P_b(x) \wedge P_a(x)$ 
are particular affirmative expressions, meant 
to be read: Some $b$ is $a$ (I in Figure 1), 
and those in the form: $\exists x.P_b(x) \wedge 
\neg P_a(x)$ are particular negative expressions, 
meant to be read: Some $b$ is not $a$ (O in Figure 1). 
\subsection{Conversion}  
There are three conversion rules: 
one for universal negative, one for 
universal affirmative, and one 
for particular affirmative. While 
Fregean predicate logic seems to 
work except for the second one, 
actually it does not: Aristotle's logic, 
as I have shown, has three truth values, 
and classical logic returns 
$\mathbb{T}$ or $\mathbb{F}$ when 
$\mathbb{U}$ is required. 
\begin{enumerate}
    \item \textbf{Universal} 
        \begin{enumerate}[label=(\Alph*)] 
            \item Negative.
                \begin{enumerate}
                    \item $(\forall x.P_1(x)
                \supset \neg P_2(x)) \supset 
                (\forall x.P_2(x) \supset  
                \neg P_1(x))$. 
             \item Example: if no pleasure is good, 
                 then no good will be pleasure. 
         \end{enumerate}  
     \item Affirmative.  
         \begin{enumerate} 
             \item $(\forall x.P_1(x) \supset 
                 P_2(x)) \supset 
                 (\exists x.P_2(x) \wedge P_1(x))$. 
             \item Example: if every pleasure is good, 
                 some good must be pleasure.  
         \end{enumerate}
        \end{enumerate} 
    \item \textbf{Particular} 
        \begin{enumerate}[label=(\Alph*)] 
            \item Affirmative. 
                \begin{enumerate} 
                    \item $(\exists x.P_1(x) \wedge 
                        P_2(x)) \supset 
                        (\exists x.P_2(x) \wedge 
                        P_1(x))$. 
                    \item Example: if some pleasure is good, 
                        then some good will be pleasure.
                \end{enumerate}
        \end{enumerate}
\end{enumerate}
\subsection{Syllogisms: three figures}  
There are three figures in Aristotle's syllogisms 
as shown below. The first figure displays syllogisms 
    for three terms $a, b$ and $c$ when 
    $b$ is the middle term which 
    is a predicate for $c$ (minor extreme), and which is predicated 
    by $a$ (major extreme). The second figure 
    displays syllogisms 
    for three terms $a, b$ and $c$ when 
    $a$ is the middle term which 
    is a predicate both for $b$ and $c$
    and when $b$ (major extreme) is closer in relation than $c$ (minor extreme) to $a$. The third figure displays 
    syllogisms for three terms $a, b$ and $c$ when 
    $c$ is the middle term that is predicated both 
    by $a$ and $b$ and when 
    $a$ (major extreme) is further in relation than $b$ 
    (minor extreme) to $c$. While again, if 
    every sentence is to be evaluated of truth/falsehood, Fregean predicate logic does appear to return a correct 
    truth value, Aristotle's logic 
    is properly three-valued. 
    Predicate gradual logic returns a matching 
    truth value out of $\mathbb{T}, \mathbb{F}$ 
    and $\mathbb{U}$ to each expression below. 
\begin{enumerate} 
    \item \textbf{First figure}{\ }
        \begin{enumerate}[label=(\Alph*)] 
            \item Universal affirmative in major and 
                minor: \\
                    $(\forall x.P_b(x) \supset P_a(x))
                \wedge (\forall x.P_c(x) \supset 
                P_b(x)) \supset 
                (\forall x.P_c(x) \supset P_a(x))$. 
    \item Universal negative in major and universal
        negative in minor: \\
             $(\forall x.P_b(x) \supset \neg P_a(x)) 
        \wedge (\forall x.P_c(x) \supset P_b(x)) 
        \supset (\forall x.P_c(x) \supset \neg P_a(x))$. 
   \item Universal affirmative in major and
       particular affirmative in minor: \\
            $(\forall x.P_b(x) \supset P_a(x)) 
               \wedge (\exists x.P_c(x) \wedge P_b(x))
               \supset (\exists x.P_c(x) \wedge P_a(x))$.
   \item Universal negative in major and particular 
       affirmative in minor: \\
            $(\forall x.P_b(x) \supset \neg P_a(x))
               \wedge (\exists x.P_c(x) \wedge P_b(x)) 
               \supset (\exists x.P_c(x) \wedge \neg 
               P_a(x))$.  
    \end{enumerate} 
\item \textbf{Second figure}{\ } 
    \begin{enumerate}[label=(\Alph*)] 
        \item Universal negative in major 
            and universal affirmative in minor:\\
                $(\forall x.P_b(x) \supset \neg P_a(x)) \wedge (\forall x.P_c(x) \supset P_a(x)) 
                    \supset (\forall x.P_c(x) \supset \neg P_b(x))$.
        \item Universal affirmative in major 
            and universal negative in minor: \\
                 $(\forall x.P_b(x) \supset P_a(x))
                    \wedge (\forall x.P_c(x) \supset 
                    \neg P_a(x)) \supset 
                    (\forall x.P_c(x) \supset 
                    \neg P_a(x))$.  
        \item Universal negative in major 
            and particular affirmative in minor: \\
                 $(\forall x.P_b(x) \supset 
                    \neg P_a(x)) \wedge 
                    (\exists x.P_c(x) \wedge P_a(x)) 
                    \supset (\exists x.P_c(x) \wedge 
                    \neg P_b(x))$.  
        \item Universal affirmative in major 
            and particular negative in minor: \\
                 $(\forall x.P_b(x) \supset P_a(x))
                    \wedge (\exists x.P_c(x) \wedge 
                    \neg P_a(x)) \supset 
                    (\exists x.P_c(x) \wedge \neg 
                    P_b(x))$. 
    \end{enumerate}  
\item \textbf{Third figure}{\ } 
    \begin{enumerate}[label=(\Alph*)] 
        \item Universal affirmative in major 
            and minor: \\
                 $(\forall x.P_c(x) \supset P_a(x))
                    \wedge (\forall x.P_c(x) \supset 
                    P_b(x)) \supset 
                    (\exists x.P_b(x) \wedge P_a(x))$.
        \item Universal negative in major 
            and universal affirmative in minor: \\
                $(\forall x.P_c(x) \supset 
                    \neg P_a(x)) \wedge 
                    (\forall x.P_c(x) \supset P_b(x))
                    \supset (\exists x.P_b(x) \wedge 
                    \neg P_a(x))$.  
        \item Particular affirmative in major 
            and universal affirmative in minor: \\
                $(\exists x.P_c(x) \wedge P_a(x)) 
                    \wedge (\forall x.P_c(x) \supset 
                    P_b(x)) \supset 
                    (\exists x.P_b(x) \wedge P_a(x))$. 
        \item Universal affirmative in major 
            and particular affirmative in minor: \\
                $(\forall x.P_c(x) \supset P_a(x))
                   \wedge (\exists x.P_c(x) \wedge 
                    P_b(x)) \supset 
                    (\exists x.P_b(x) \wedge P_a(x))$.
        \item Particular negative in major 
            and universal affirmative in minor: \\
                $(\exists x.P_c(x) \wedge 
                    \neg P_a(x)) \wedge 
                    (\forall x.P_c(x) \supset 
                    P_b(x)) \supset 
                    (\exists x.P_b(x) \wedge 
                    \neg P_a(x))$. 
        \item Universal negative in major 
            and particular affirmative in minor: \\
                 $(\forall x.P_c(x) \supset 
                    \neg P_a(x)) \wedge 
                    (\exists x.P_c(x) \wedge P_b(x)) 
                    \supset (\exists x.P_b(x) \wedge 
                    \neg P_a(x))$.  
    \end{enumerate}
\end{enumerate}

\section{Conclusion}  
I presented a new logic and a new approach, which is 
for now a novel principle with bright outlooks, 
expected to offer a new perspective 
in formal logic. I demonstrated 
against the modern account of existential import, 
corroborating former anticipation \cite{Hart51,Strawson52}. 
I showed that Aristotle's syllogisms as well as 
conversion are realisable in 
predicate gradual logic. 
\bibliography{references}
\bibliographystyle{plain} 
\end{document}